\documentclass{article} 
\usepackage{iclr2026_conference,times}

\usepackage{hyperref}
\usepackage{url}
\usepackage[utf8]{inputenc} 
\usepackage[T1]{fontenc}    
\usepackage{url}            
\usepackage{booktabs}       
\usepackage{amsfonts}       
\usepackage{nicefrac}       
\usepackage{microtype}      
\usepackage{xcolor}         
\usepackage{makecell}
\usepackage{subcaption}
\usepackage{caption}
\usepackage{fontawesome5}

\usepackage{graphicx}

\usepackage{amsmath, amsthm, amssymb}
\usepackage{amsfonts,bm,bbm}
\usepackage{apptools}

\usepackage{thmtools,thm-restate}
\newtheorem{theorem}{Theorem}
\newtheorem{lemma}[theorem]{Lemma}
\newtheorem{corollary}[theorem]{Corollary}

\newtheorem{claim}{Claim}

\AtAppendix{\counterwithin{lemma}{section}}
\AtAppendix{\counterwithin{theorem}{section}}
\AtAppendix{\counterwithin{claim}{section}}
\AtAppendix{\counterwithin{assumption}{section}}
\AtAppendix{\counterwithin{corollary}{section}}

\AtAppendix{%
}

\def\1{\bm{1}}

\def\vbeta{{\bm{\beta}}}
\def\vxi{{\bm{\xi}}}

\def\ve{{\bm{e}}}

\def\vv{{\bm{v}}}

\def\vx{{\bm{x}}}
\def\vy{{\bm{y}}}

\def\mH{{\bm{H}}}
\def\mI{{\bm{I}}}

\def\mM{{\bm{M}}}

\def\mS{{\bm{S}}}

\def\mX{{\bm{X}}}

\def\mZ{{\bm{Z}}}

\def\mSigma{{\bm{\Sigma}}}

\DeclareMathAlphabet{\mathsfit}{\encodingdefault}{\sfdefault}{m}{sl}
\SetMathAlphabet{\mathsfit}{bold}{\encodingdefault}{\sfdefault}{bx}{n}

\def\gD{{\mathcal{D}}}

\def\gF{{\mathcal{F}}}

\newcommand{\E}{\mathbb{E}}
\renewcommand{\P}{\mathbb{P}}

\newcommand{\R}{\mathbb{R}}

\DeclareMathOperator{\Tr}{Tr}

\newcommand{\norm}[1]{\left\|#1\right\|}
\newcommand{\ind}{\mathbbm{1}}

\newcommand{\diag}{\mathrm{diag}}

\newcommand{\polylog}{\mathrm{polylog}}

\newcommand{\tldTheta}{\widetilde{\Theta}}
\newcommand{\tldOmega}{\widetilde{\Omega}}
\newcommand{\tldO}{\widetilde{O}}

\newcommand{\mXlek}{\mX_{\le k}}

\newcommand{\plek}{p_{\le k}}
\newcommand{\pgtk}{p_{> k}}
\newcommand{\vxlek}{\vx_{\le k}}
\newcommand{\vxgtk}{\vx_{> k}}
\newcommand{\mSigmalek}{\mSigma_{\le k}}
\newcommand{\vbetalek}{\vbeta_{\le k}}

\newcommand{\hvbeta}{\hat{\vbeta}}
\newcommand{\hvbetalek}{\hvbeta_{\le k}}

\newcommand{\hbeta}{\hat{\beta}}

\title{Memorizing Long-tail Data Can Help Generalization Through Composition}

\author{Mo Zhou\thanks{ Equal contribution.} \ $^1$, Haoyang Ma\footnotemark[1] \ $^2$, Rong Ge $^2$  \\
$^1$ University of Washington, $^2$ Duke University\\
\texttt{mozhou17@cs.washington.edu, haoyang.ma@duke.edu, rongge@cs.duke.edu}
}

\iclrfinalcopy 
\begin{document}

\maketitle

\begin{abstract}
    Deep learning has led researchers to rethink the relationship between memorization and generalization. In many settings, memorization does not hurt generalization due to implicit regularization and may help by memorizing long-tailed examples. In this paper, we consider the synergy between memorization and simple composition --- the ability to make correct prediction on a combination of long-tailed features. Theoretically, we show that for a linear setting, memorization together with composition can help the model make correct predictions on rare test examples that require a combination of long-tailed features, even if such combinations were never observed in the training data. Experiments on neural network architecture on simple data show that the theoretical insight extends beyond the linear setting, and we further observe that the composition capability of the model depends on its architecture. 
    \footnote{Code available at \url{https://github.com/mhy-666/long-tail-memorization-composition}.}
    \footnote{arXiv v2 fixes several minor errors in the proofs and improves Theorems 4 and 5.}
\end{abstract}
\vspace{-7pt}
\section{Introduction}
\vspace{-5pt}

The relationship between memorization and generalization has always been an intriguing topic in deep learning. It has long been known that neural networks used for supervised learning can memorize noisy or even random labels \citep{zhang2017understanding}, and recent large language models can still memorize long text despite not being overparametrized \citep{carlini2019secret,carlini2022quantifying}. Conventional wisdom from statistical learning theory suggests that memorization might be detrimental to generalization performance, yet neural networks often generalize well with memorization, sometimes even better compared to networks that do not memorize.

Trying to understand this relationship from a theoretical perspective has led to the idea of implicit regularization and benign overfitting \citep{belkin2019reconciling,bartlett2020benign,belkin2020two,hastie2022surprises,gunasekar2017implicit}, where the training process and architecture choices of neural networks prefer certain solutions that have good generalization despite memorizing the training data. Another interesting line of work \citep{feldman2020does,feldman2020neural} demonstrated that memorization can actively help generalization by capturing long-tail behaviors in the training data. If the test data contain examples that are similar to training examples, it is more likely for the neural network to make correct predictions.

However, the similarity between the test and training data alone cannot explain all the benefits of memorization. Recent large models exhibit several surprising properties including {\em composition}  \--- the ability to produce correct results by combining two or more pieces of information learned during training. Composition has received a lot of attention \citep{hupkes2020compositionality,baroni2020linguistic}. Composition allows models to perform tasks that may not appear in the training data, resulting in stronger notions of out-of-distribution (OOD) generalization. Intuitively, memorizing long-tail examples and composition are synergistic: a model that is capable of composition may leverage two or more memorized long-tail examples to create/predict something that is highly unlikely in the original training data. In this paper, we try to demonstrate this behavior on simplistic models.

Our theoretical approach focuses on a linear setting where features appear in different frequencies, and each sample consists of a sparse combination of features.  
As the model is linear, it automatically achieves simple composition --- if individual features are learned correctly, their combinations can also be predicted correctly. We give guarantees for out-of-distribution generalization where the test data contain examples that may be combinations of multiple infrequent features and have very low probability of appearing in the training set.

To complement our theoretical results, we construct simple examples that require simple composition \---- prediction of the sum of digits for 3-digit numbers. Our experiment shows that in such cases one can indeed train simple networks that are capable of composition and memorization helps with both in-distribution and out-of-distribution (OOD) generalization. We also observe that the model architecture plays an important role in determining whether the model has composition capabilities.

\vspace{-3pt}
\subsection{Our results}
\vspace{-3pt}
Informally, our theoretical model considers a data distribution $\gD$ over vectors $\vx$, where each feature $x_i$ appears independently with probability $p_i$. These probabilities decay as $p_1 \ge p_2 \ge \ldots \ge p_d$ (for example, $p_i \sim i^{-\alpha}$ follows a power-law decay), and each data is a linear combination of $s$ features. As a result, features with larger $p_i$ tend to appear frequently in the data, while those with smaller $p_i$ may appear only a few times or not at all. We refer to the former as common features and the latter as long-tail features.

Our goal is to learn a target model of the form $f_*(\vx) = \langle \vbeta_*, \vx \rangle$ using $n$ training samples of the form $y = f_*(\vx) + \xi$, where the label noise $\xi \sim \mathcal{N}(0, \sigma^2)$ and the noise level $\sigma \ge 0$. We focus on the minimum $\ell_2$-norm solution, which memorizes the training set in our overparametrized setting. 

We show that under this setup, the model can not only reliably learn common features that appear at least $\tilde{\Omega}(1)$ times\footnote{We use $\tilde{O},\tilde{\Omega}, \tilde{\Theta}$ to hide polylog factors.}, but also recover certain long-tail features that appear only once. This enables the model to perform well when tested on combinations of such features that were never seen together during training:
\begin{theorem}[Informal]
    Let $\widetilde{\gD}$ be a test distribution that differs from the training distribution $\gD$ only in that each sample from $\widetilde{\gD}$ contains a composition of $t$ long-tail features, each of which appeared exactly once in the training data. Then, for the minimum $\ell_2$-norm solution the test loss (expected loss on $\gD$) remains small, and the out-of-distribution loss (expected loss on $\widetilde{\gD}$) is at most $\sigma^2 t$.
\end{theorem}

The main challenge is that the frequency $p_i$ of common features can be as low as $\tilde{\Theta}\left(\frac{1}{n}\right)$, resulting in a high variance. This issue is even worse for long-tail features, which appear even less frequently. As a result, standard concentration bounds are insufficient in this setting. To address this, we leverage the combinatorial structure of the data matrix $\mX$ and show that, under a suitable decay of $p_i$, most data points contain at most one long-tail feature. This structural insight enables us to establish meaningful learning guarantees, even when some features are extremely rare and variances are large.

Empirically, we first verify the theoretical analysis in the linear setting. We then show that similar phenomena happen for simple neural networks. More concretely, we consider a ``heavy-tailed'' version of MNIST dataset where digits appear in decreasing frequencies, with 9 being the rarest. To introduce compositionality, each input consists of 3 randomly sampled MNIST images, and the task is to predict the sum of the 3 digits. 

We evaluate the composition capability of model through its OOD performance on test examples that contains at least one 9, with the other two numbers sampled uniformly. Most of the test cases involve 3-digit combinations that never appear during training. We observe that ResNet models exhibit compositionality when they process each digit individually to extract features, and then aggregated these features by a simple network --- even if the network has not seen a particular 3-digit combination, the prediction is still accurate if it has seen each digit enough number of times. In contrast, a ResNet that processes the 3 images together as a single input performs significantly worse. 

To further investigate memorization effect on generalization, we insert images from Omniglot in the training data, where each Omniglot class only appear {\em exactly once}. Experiments show that models with compositional ability can correctly predict on test data that involves two Omniglot images, provided that the model memorizes training data. However, models that either do not memorize due to a large weight decay or lack compositional ability perform significantly worse.

\subsection{Related Works}

\paragraph{Memorization and long‑tail data}
Real-world data often follow heavy-tailed distributions, where rare examples capture unique sub-concepts \citep{zhu2014capturing}. 
The works most related to ours are \citet{feldman2020does} and \citet{feldman2020neural}. These papers demonstrated that memorization is necessary to achieve near-optimal generalization on long-tail data. Our work builds on this insight and takes it further, suggesting that memorization can also support the composition of long-tail data and features.

\paragraph{Memorization in learning}
Researchers have discovered that deep neural networks can fit (i.e., ``memorize'') arbitrary labels and still generalize, challenging the classical bias–variance intuition \citep{zhang2017understanding}. Since then, there has been growing interest in understanding memorization from various perspectives (see the survey by \cite{wei2024memorization} for a comprehensive overview). 

Various forms of memorization have been explored in different contexts, including label memorization in supervised learning \citep{feldman2020does,feldman2020neural}, counterfactual memorization \citep{zhang2023counterfactual}, and exact memorization in language models \citep{tirumala2022memorization}. These findings have raised privacy concerns about the unintended memorization of sensitive data \citep{carlini2019secret}. In response, forgetting mechanisms and unlearning algorithms have been developed to remove such content without full retraining (e.g., \cite{cao2015towards,garg2020formalizing}). In contrast, our work focuses on the positive side of memorization, similar to \cite{feldman2020does}, showing how it can support compositional learning over long-tail data distributions.

\paragraph{Compositional Generalization}
Generalizing to unseen compositions of known data or features has a long history \citep{fodor1988connectionism} and remains a significant challenge. In recent years, benchmarks such as SCAN \citep{lake2018generalization} and CoGS \citep{kim2020cogs} have been introduced to evaluate models' compositional abilities. Extensive analyses using these benchmarks have revealed that some language and vision-language models struggle with aspects like word order sensitivity, relationship recognition, or counting \citep{hessel2021effective,parcalabescu2021valse,yuksekgonul2022and}. Nonetheless, recent work suggests that current language models do exhibit some capacity for compositional skill \citep{bubeck2023sparks, zhao2025modelslearnskillcomposition}.

To improve compositional generalization, researchers have proposed various methods, including data augmentation \citep{andreas2019good}, architectural designs \citep{andreas2016neural,hudson2018compositional}, and meta-learning approaches \citep{conklin2021meta}. Efforts to better understand compositional generalization continue in recent years \citep{wiedemer2023compositional,tang2025explainable,arora2023theory,lippl2024does}. Our work builds on this line of inquiry, but we focus more on the connection between memorization and compositionality. 

\section{Preliminaries}\label{sec:Preliminaries}

In this section, we set up the details of our theoretical model.

\paragraph{Notation}
Let $[n] = {1, 2, \ldots, n}$. Vectors and matrices are denoted in bold. For a vector $\vbeta \in \R^d$ and a subset $A \subseteq [d]$, define $\vbeta_A := \sum_{i \in A} \beta_i \ve_i$, keeping entries in $A$ and setting others to 0, where ${\ve_i}$ is the standard basis. We use standard notation $O$, $\Omega$, $\Theta$, $\lesssim$, $\gtrsim$ to hide constants, and $\tldO$, $\tldOmega$, $\tldTheta$ to also hide logarithmic factors like $\log(d)$ and $\log(n)$.

\paragraph{Target Model}
We aim to learn the following linear model:
\[
    f_*(\vx) = \langle\vbeta^*,\vx\rangle.
\]
Here we assume largest coefficient is bounded $\max_i |\beta_i^*| = \Theta(1)$.

\paragraph{Data}
Our data model is designed to capture the common long-tail distribution of features in real-world data \citep{zhu2014capturing}. We assume each data point has only a few active features, with feature frequencies varying according to a long-tail pattern. For instance, in image classification each image may contain only a few objects, while the overall feature set is large and follows a long-tail distribution.

Formally, the data vectors $\vx = (x_1, \ldots, x_d)^\top$ are drawn from the distribution $\gD(\{p_i\})$ defined as follows: given probability $1\ge p_1\ge\ldots\ge p_d>0$, each feature $x_i \in \{0, 1, -1\}$ is independently sampled from a Bernoulli distribution with parameter $p_i$, followed by a random sign flip:
\begin{align*}
    \P(x_i = 0) = 1 - p_i, \quad
    \P(x_i = 1) = \P(x_i = -1) = \frac{p_i}{2}
\end{align*}
Let $\mSigma = \E[\vx \vx^\top] = \diag(p_1, \ldots, p_d)$ be the covariance matrix, and we write $\gD(\{p_i\})$ as $\gD(\mSigma)$. 

We generate $n$ data points $\{(\vx_i, y_i)\}_{i=1}^n$ i.i.d. according to this process with independent label noise $\xi$:
\begin{align}\label{eq: data generation}
    y_i = f_*(\vx_i) + \xi_i, 
    \quad \text{where } 
    \vx_i \sim \gD(\mSigma),\ \xi_i \sim \mathcal{N}(0, \sigma^2).
\end{align}

In many cases, we assume that the feature probabilities $p_i$ follow a power-law decay:
\begin{align}\label{eq: data power law}
    p_i := \min\{1, s \cdot i^{-\alpha}/Z_\alpha\}
\end{align}
for some constant $\alpha>0$ with normalization constant $Z_\alpha=\sum_{i\le d}i^{-\alpha}$. The truncation at 1 is only to make sure $p_i\in[0,1]$ as a valid probability. Note that $\E[\norm{\vx}_2^2]\approx s$ (ignoring the truncation) so each data $\vx$ will be roughly $s$-sparse.  In the rest of the paper, for a data $\vx$ we will often call the top-$k$ entries $\vxlek$ as common features and remain entries $\vxgtk$ as (long-)tail features for a threshold $k$ chosen later. Also denote $\pgtk := \sum_{i>k}p_i$ as the tail probability and $\plek := \sum_{i\le k}p_i$ as the head probability. The tail probability can be bounded by the claim below.
\begin{claim}\label{claim: pk}
    For $p_i$ follows power law decay \eqref{eq: data power law} with $\alpha=1+c_\alpha$ for any constant $c_\alpha>0$ and $p_k <1$ (i.e., $k\ge\Theta(s^{1/\alpha})$), we have $\pgtk = \Theta(sk^{1-\alpha})$ and $p_k = \Theta(sk^{-\alpha})$.
\end{claim}

Our data model is intended to reflect the long-tail structure, and the power-law decay assumption is common in prior theoretical analyses (e.g., \citep{bartlett2020benign, hastie2022surprises, lin2024scaling}).

\paragraph{Choice of parameters}
Let $c_{\log}=\polylog(d)$ be a large enough polylog factor. We assume $s < c_{\log} < k < n/c_{\log} < d/c_{\log}^2$ and $n^c>s+\ln d$ for any constant $c>0$. A concrete example is $s=\Theta(1)$, $n=d^{c}$ for constant $c\in(0,0.99)$ and $c_{\log}<k<n/c_{\log}$. Casual readers may assume $s\ll \polylog(d) \ll k \ll n \ll d$. 
The threshold $k$ between common and long-tail features will be chosen later. Intuitively we require $np_k = \tilde{\Theta}(1)$ so common features appear at least polylog times.

\paragraph{Algorithm and loss}
We use a linear model and run gradient descent on square loss. That is
\[
    L(\vbeta) 
    = \frac{1}{n}\sum_{i=1}^n(\langle\vbeta,\vx_i\rangle-y_i)^2
    = \frac{1}{n}\norm{\mX\vbeta - \vy}_2^2. 
\]

We run gradient descent from 0 initialization, so the final estimator is
\[
    \hat{\vbeta}=(\mX^\top\mX)^\dagger\mX^\top\vy.
\]
This estimator is also known as the min-$\ell_2$-norm interpolator/solution. Similar overparametrized linear models and kernel-models are standard in previous theoretical works 
(e.g., \cite{bartlett2020benign, belkin2019reconciling,jacot2018neural}).

We focus on the overparametrized regime where $d\gg n$. In the noiseless case $\sigma=0$, such interpolator memorizes all training data. In noisy case $\sigma>0$, the minimum training loss may be above 0 even in such overparametrized setting (because there may be many samples that only use a smaller number of features), so the solution may not interpolate the data. Still, it represents a global minimum with the smallest $\ell_2$ norm among all such minima.

Since both our data generation process and predictor are linear, composition happens automatically --- if a model learns the $\beta$ values for a subset of features accurately, it will perform well for any data that uses a combination of features from this subset.

\paragraph{Performance Measurements: In- and Out-of-Distribution Test Loss}
For in-distribution performance, we use the standard test loss:
\[
    \E_{x\sim\gD}[(\hvbeta^\top\vx - \vbeta^{*\top}\vx)^2]. 
\]

To evaluate the model's composition capability (its ability to perform well on rare test examples that are combinations of training features never appear in training data), we measure its OOD performance using the test loss defined below.

Let 
$\gF = \{j\in[d]: \exists i\in[n]\ s.t.\ (\vx_i)_j\ne 0\}$ be the set of features that appear in the training data. Let $\widetilde{\gF} \subseteq \gF$ be a subset of these features. We define $\gD_{\widetilde{\gF}}$ as a modified distribution of $\gD$, where each feature $x_i$ in $\widetilde{\gF}$ is guaranteed to be non-zero in every sample, while the rest of the data is generated according to the original distribution $\gD$: each feature $x_i$ not in $\widetilde{\gF}$ appears with its original frequency $p_i$, and non-zero values are sampled uniformly from ${\pm1}$. 

For example, $\gD_{\{2k,4k\}}$ denotes a distribution where $x_{2k}, x_{4k} \ne 0$. While the model might encounter training examples where only $x_{2k}$ or $x_{4k}$ is non-zero, it is very likely that it does not see examples where both are non-zero simultaneously. The OOD test loss is defined as:
\[
\E_{x\sim\gD_{\widetilde{\gF}}}[(\hvbeta^\top\vx - \vbeta^{*\top}\vx)^2].
\]

\section{Memorization helps out-of-distribution generalization}\label{sec: main result}

In this section, we present our main results, which show that memorization can help with generalization by composition of long-tail features, as measured by test loss and OOD loss. We present results for both the noiseless case ($\sigma = 0$) and the noisy case ($\sigma > 0$), respectively.

\subsection{Noiseless case}
We start with noiseless case. We show below the test performance of the estimator $\hvbeta$ that goes to 0 when $n$ goes large. The intuition behind the proof is that we can recover all feature that shows up at least $\Theta(\ln^2 d)$ times, so their estimation error goes to 0. 
\begin{restatable}[Test loss]{theorem}{thmtestloss}\label{thm: test loss}
    Suppose data generated from \eqref{eq: data generation} with noise $\sigma=0$, then with probability at least 0.99 over the data generation process, the following hold with $k$ satisfying $np_k = \Theta(\ln^2 d)$ and $\pgtk\le 1-c_p$ with any constant $c_p$:   
    \[
        \E_{x\sim\gD}[(\hvbeta^\top\vx - \vbeta^{*\top}\vx)^2]
        \lesssim \pgtk
            + np_k\pgtk^2
    \]    
    When data follows power law decay \eqref{eq: data power law} with $\alpha=1+c_\alpha$ with any constant $c_\alpha>0$
    \[
        \E_{x\sim\gD}[(\hvbeta^\top\vx - \vbeta^{*\top}\vx)^2]  
        \lesssim s\left(\frac{\ln^2 d}{ns}\right)^{1-\frac{1}{\alpha}}
    \]
\end{restatable}

Moreover, we establish the following OOD generalization guarantee, which reflects the model's compositional ability and shows that it can still perform well despite memorizing all training data. As shown in the proof, this is made possible by the composition of features that the model memorizes.

\begin{restatable}[OOD performance]{theorem}{thmood}\label{thm: ood}
    Under same condition as Theorem~\ref{thm: test loss}, then with probability at least 0.99 over the data generation process, there is a feature set $\hat{\gF}\subseteq\gF$ with size $|\hat{\gF}|\ge (1-\Theta(\pgtk\ln^2 d))|\gF|$ that contains almost all features show up in data such that for any feature set $\widetilde{\gF}\subseteq\hat{\gF}$ we have
    \begin{align*}    
        \E_{x\sim\gD_{\widetilde{\gF}}}[(\hvbeta^\top\vx - \vbeta^{*\top}\vx)^2]
        \lesssim
        \pgtk
            + np_k\pgtk^2
    \end{align*}
    When data follows power law decay \eqref{eq: data power law} with $\alpha=1+c_\alpha$ with any constant $c_\alpha>0$ we have $|\hat{\gF}|\ge (1-\Theta\left(\frac{s^{\frac{1}{\alpha}}(\ln d)^{4-\frac{2}{\alpha}}}{n^{1-\frac{1}{\alpha}}}\right))|\gF|=(1-o(1))|\gF|$ and
    \begin{align*}    
        \E_{x\sim\gD_{\widetilde{\gF}}}[(\hvbeta^\top\vx - \vbeta^{*\top}\vx)^2]
        \lesssim&
        s\left(\frac{\ln^2 d}{ns}\right)^{1-\frac{1}{\alpha}}.
    \end{align*}
\end{restatable}
The proof ideas for both results are similar. We show that all features that appear at least $\Theta(\ln^2 d)$ times can be recovered, as well as a significant fraction of long-tail features that appear less frequently in the training data. Intuitively, this is the best we can hope for, since common features may appear only $\omega(1)$ times, and long-tail features may appear just once. We provide more details in Section~\ref{sec: proof overview}.

These results suggest that memorization can be useful for learning compositions involving long-tail features. This may offer insight into why memorization is necessary in certain learning problems that require compositional ability. We further validate our findings in the experimental results (Section~\ref{sec: experiments}).

\subsection{Noisy case}
We now extend the above noiseless case ($\sigma=0$) to the noisy case $(\sigma>0)$. Similarly, we show the test loss is small.
\begin{restatable}[Test loss]{theorem}{thmtestlossnoisy}\label{thm: test loss noisy}
    Suppose data generated from \eqref{eq: data generation} with noise $\sigma\le 1$, with probability at least 0.99 over the data generation process, the following hold with $k$ satisfying $np_k = \Theta(\ln^2 d)$, $\pgtk\le 1-c_p$ with any constant $c_p$ and $n\pgtk^2 <c$ for a small enough constant $c$:
    \begin{align*}    
        \E_{x\sim\gD}[(\hvbeta^\top\vx - \vbeta^{*\top}\vx)^2]
        \lesssim& \sum_{j>k,j\not\in\gF} p_j (\beta_j^*)^2
            + \sigma^2 \left(\frac{k}{n}
            +\sum_{j>k,j\in\gF}p_j\frac{\ln d}{|S_{1,j}|}\right)
    \end{align*}
    where $\gF$ is the set of feature appear in training data and $S_{1,j}$ is the set of data that contains feature $j$.
    
    When data follows power law decay \eqref{eq: data power law} with $\alpha=2+c_\alpha$ with any constant $c_\alpha>0$
    \[
        \E_{x\sim\gD}[(\hvbeta^\top\vx - \vbeta^{*\top}\vx)^2] 
        \lesssim s\left(\frac{\ln^2 d}{ns}\right)^{1-\frac{1}{\alpha}}
            + \sigma^2 \left(\frac{k}{n}
            +\sum_{j>k,j\in\gF}p_j\frac{\ln d}{|S_{1,j}|}\right)
    \]
\end{restatable}
It worth to interpret the RHS of the bound. Here $\sum_{j>k,j\not\in\gF} p_j (\beta_j^*)^2$ reflects the error due to the unseen feature in the training; $\sigma^2k/n$ reflects the standard sparse linear regression error of fitting the common part; $(\sigma^2\ln d)/|S_{1,j}|$ reflects the error caused by $j$-th feature, and decreases as $|S_{1,j}|$ increases.

Similar to the noiseless case, we have the following OOD guarantee that measure model's compositional ability: the error by $j$-th feature is roughly $(\sigma^2 \ln d)/|S_{1,j}|$, where $|S_{1,j}|$ is the number of features $j$ present in the training data.
\begin{restatable}[OOD performance]{theorem}{thmoodnoisy}\label{thm: ood noisy}
    Under same condition as Theorem~\ref{thm: test loss noisy}, then with probability at least 0.99 over the data generation process, consider any feature set $\widetilde{\gF}\subseteq\gF\setminus\gF_0$ contains any long tail features show up in data, we have
    \begin{align*}    
        \E_{x\sim\gD_{\widetilde{\gF}}}[(\hvbeta^\top\vx - \vbeta^{*\top}\vx)^2]
        \lesssim& \sum_{j>k,j\not\in\gF} p_j (\beta_j^*)^2
            + \sigma^2\left(
            \frac{k}{n}
            + \sum_{j>k,j\in\gF\setminus\widetilde{\gF}}p_j\frac{\ln d}{|S_{1,j}|}
            +\sum_{j>k,j\in\widetilde{\gF}}\frac{\ln d}{|S_{1,j}|}\right),
    \end{align*}
    where $\gF$ is the set of feature appear in training data and $S_{1,j}$ is the set of data that contains feature $j$.
    
    When data follows power law decay \eqref{eq: data power law} with $\alpha=2+c_\alpha$ with any constant $c_\alpha>0$
    \[
        \E_{x\sim\gD_{\widetilde{\gF}}}[(\hvbeta^\top\vx - \vbeta^{*\top}\vx)^2]
        \lesssim s\left(\frac{\ln^2 d}{ns}\right)^{1-\frac{1}{\alpha}}
            + \sigma^2\left(
            \frac{k}{n}
            + \sum_{j>k,j\in\gF\setminus\widetilde{\gF}}p_j\frac{\ln d}{|S_{1,j}|}
            +\sum_{j>k,j\in\widetilde{\gF}}\frac{\ln d}{|S_{1,j}|}\right).
    \]
\end{restatable}

The terms on the RHS are the same as those in Theorem~\ref{thm: test loss noisy}. The only additional term arises from the OOD dataset, which forces the features in $\widetilde{\gF}$ to always appear. In the worst case, when all $|S_{1,j}|=1$, the $\sigma^2$-dependent terms scale as $\sigma^2|\widetilde{\gF}|\ln d$. Thus, the error remains proportional to the noise level $\sigma^2$ and can be further reduced as more training samples are observed in $S_{1,j}$.

These results suggest that the benefits of memorization for compositional learning persist even in the presence of label noise, although they require a stronger assumption on the frequency decay.

\section{Proof idea for noiseless case}\label{sec: proof overview}
In this section, we present the proof ideas for the noiseless case. The proof for the noisy case follows a similar idea, but is slightly more complex.

A common strategy in analyzing the least squares estimator is to apply matrix concentration techniques to estimate $\mX^\top \mX$ and its inverse. However, in our setting, such a matrix is not going to concentrate for long-tail features that are expected to appear less than a constant number of times. Whether a feature appears or not has huge impact on the pseudoinverse $(\mX^\top \mX)^\dag$. 
Instead of concentration, we rely on the combinatorial structure of the data matrix $\mX$, as described in the lemma below. Roughly speaking, the lemma shows that most data points contain either zero or one long-tail feature.
\begin{lemma}[Structure of $\mX$, informal of Lemma~\ref{lem: structure of X}]\label{lem: structure of X informal}
    Suppose data generated from \eqref{eq: data generation}, with probability at least 0.99, the following hold with $k$ satisfying $np_k = \Theta(\ln^2 d)$ and $\pgtk\le 1-c_p$ with any constant $c_p$:
    \begin{enumerate}
        \item There exist at least $\Theta(n(1-\pgtk))$ samples whose non-zero entries are only within first $k$ entries. That is
        \[
            |S_0|:=\left|\{i\in[n]: (\vx_i)_j = 0,\forall j > k\}\right| = \Theta(n(1-\pgtk)).
        \]
        Moreover, the dimension spanned by these samples in $S_0$ is $k$.
        
        \item There exist at least $\Theta(n\pgtk(1-\pgtk))$ samples that only have exactly one non-zero entry in long-tail entries (not in the first $k$ entries). That is
        \[
            |S_1|:=\left|\{i\in[n]: \exists \ell >k \text{ such that } (\vx_i)_\ell\ne 0 \text{ and } (\vx_i)_j = 0,\forall j > k, j\ne \ell\}\right| = \Theta(n\pgtk(1-\pgtk)).
        \]

        \item Denote the set of features show up in $S_0$ and $S_1$ as 
        \[
            \gF_0 = \{j\in[d]: \exists i\in S_0\ s.t.\ (\vx_i)_j\ne 0\},\quad
            \gF_1 = \{j\in[d]: \exists i\in S_1\ s.t.\ (\vx_i)_j\ne 0\}.
        \]
        We have the number of features shown up in $S_0$ and $S_1$ is at least 
        \[
            |\gF_0\cup\gF_1|\ge \left(1-\Theta\left(\pgtk\ln^2 d\right)\right)|\gF|
        \] 
    \end{enumerate}
\end{lemma}

Note that in the noiseless case, we can achieve 0 training loss. Thus we know every sample must have 0 training loss. From the lemma above, we can see that there are $\Theta(n(1 - \pgtk))$ samples involving only the $k$ common features. Therefore, the common component $\hvbetalek$ is uniquely determined by these data points when $k < n$.

Moreover, once the common features $\hvbetalek$ are identified, many long-tail features $\hbeta_i$, those that appear only alongside common features and not with other long-tail features, can also be exactly recovered. This is formalized in the following lemma:
\begin{lemma}[Recovery of $\beta^*$]
    Suppose data generated from \eqref{eq: data generation} with noise $\sigma=0$, for $k$ that satisfies the condition of Lemma~\ref{lem: structure of X informal}, we have
    \begin{enumerate}
        \item Common features (Top-$k$ entries) are recovered: $\hvbetalek = \vbetalek^*$
        \item Most of long tail features shown up in data $X$(last-$d-k$ entries) are recovered:
        \[
            \left|\left\{i\in[d]: \hbeta_i = \beta^*_i\right\}\right|/\left|\gF\right|
            \ge 1-\Theta(\pgtk\ln^2 d).
        \]
        \item $\norm{\hvbeta_{\gF\setminus(\gF_0\cup\gF_1)}}_2\le \norm{\vbeta^*_{\gF\setminus(\gF_0\cup\gF_1)}}_2$ and $\hbeta_i=0$ for $i\not\in\gF$.
    \end{enumerate}
\end{lemma}

Using the parameter recovery result above, we can derive the performance guarantees stated in Theorem~\ref{thm: test loss} and Theorem~\ref{thm: ood}, by applying Claim~\ref{claim: pk}.

In the noisy case, we require stronger tail probability decay assumption of data $\mX$ to show that (with high probability) every data point contains at most one long-tail feature. Even with this assumption, the key challenge for the noisy case is that the minimum $\ell_2$ norm solution may not have 0 training loss, hence we can no longer say that the training loss is 0 on every individual training samples as in the noiseless case. However, we show that even in this setting, if a long-tail feature only appear in one sample, then the training loss for that particular sample must be 0 and the long-tail feature will be estimated correctly (up to noise level $\sigma$) with high probability. Similar argument can be generalized when long-tail feature appear in few training samples. We leave the details to Appendix~\ref{appendix: noisy}.

\section{Experiment Results}
\label{sec: experiments}
In this section, we first validate our theoretical results empirically, and then we show that similar ideas can be extended to simple neural network models. All experiments, unless otherwise specified, reach near-zero training loss and therefore memorize the training data.

\subsection{Experiments on Linear Model} 
We follow the setup in Section~\ref{sec:Preliminaries} where the parameters are chosen as: $n = 1000$, $d = 10000$, $s = 5$, $np_k = 10$ and $\vbeta^*=(\beta_1^*,\ldots,\beta_d^*)^\top$ with $\beta^*_i = i^{-0.1}$, and report results averaging over 50 runs.
In this setting the most frequent ``long-tail'' features would appear 10 times in expectation. As we can see in Figure~\ref{fig:norm_greater_k_vs_alpha}, for most values of $\alpha$ the error on the long-tail features is roughly proportional to the noise level $\sigma$. This is the same as the theoretical prediction as many of these long-tail examples appear only once and their error is slightly larger than the noise, while long-tail examples that appear more often can have a smaller error.
Figure~\ref{fig:norm_less_k_vs_alpha} shows that common features are learned well when $\alpha$ is not too small.
Figure~\ref{fig:iod_test_loss}\ref{fig:ood_test_loss} shows that the in-distribution and out-of-distribution test losses are small as in Theorem~\ref{thm: test loss},~\ref{thm: ood},~\ref{thm: test loss noisy} and~\ref{thm: ood noisy}. Since the test losses for $\sigma = 0$ and $\sigma = 0.05$ are very small, their curves overlap in the plot.

We also observe that when $\alpha$ is small ($\alpha =1$),  both the in-distribution and out-of-distribution test losses become higher, and the feature errors also increase. This is expected because in such cases the long-tail features appear more often in the training data, not often enough for concentration bounds but often enough to weaken the combinatorial structure our analysis relies on.

\begin{figure}[t]
    \centering
    \subcaptionbox{Average Error on Common Features \\ with Different $\alpha$ and $\sigma$\label{fig:norm_less_k_vs_alpha}}[0.24\textwidth]{%
        \includegraphics[width=\linewidth]{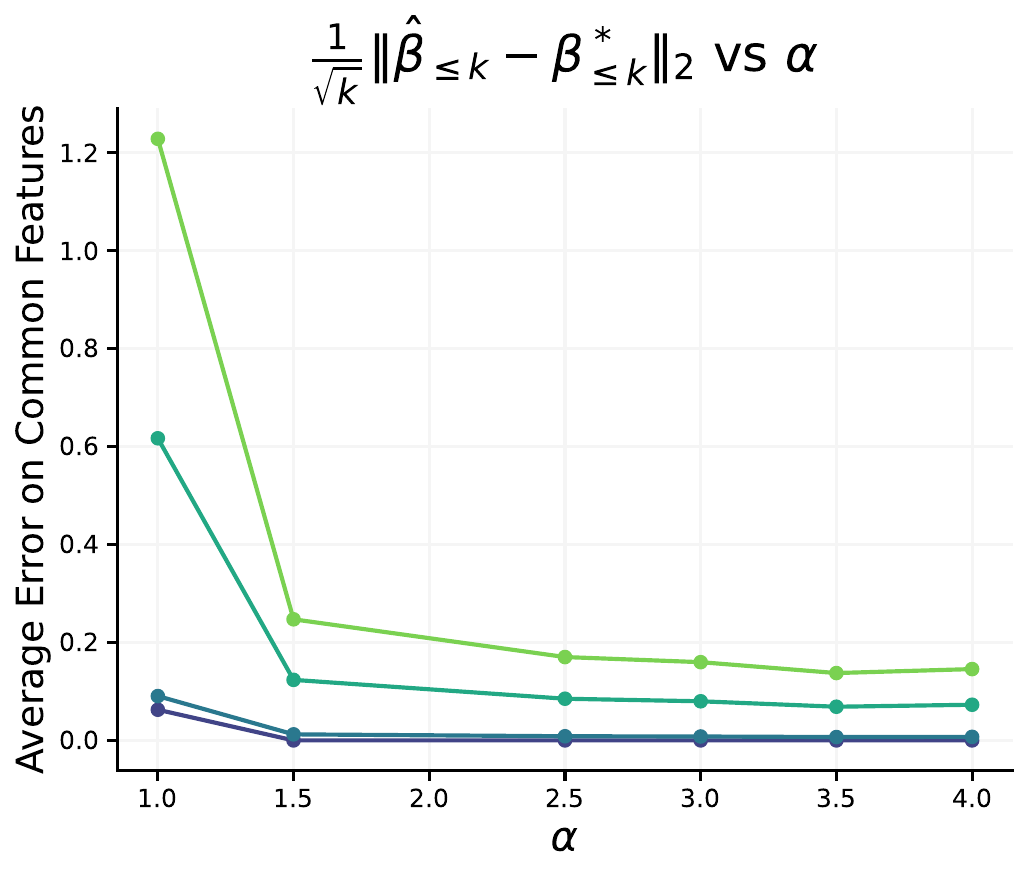}}
    \subcaptionbox{Average Error on Long-tail Features \\ with Different $\alpha$ and $\sigma$\label{fig:norm_greater_k_vs_alpha}}[0.24\textwidth]{%
        \includegraphics[width=\linewidth]{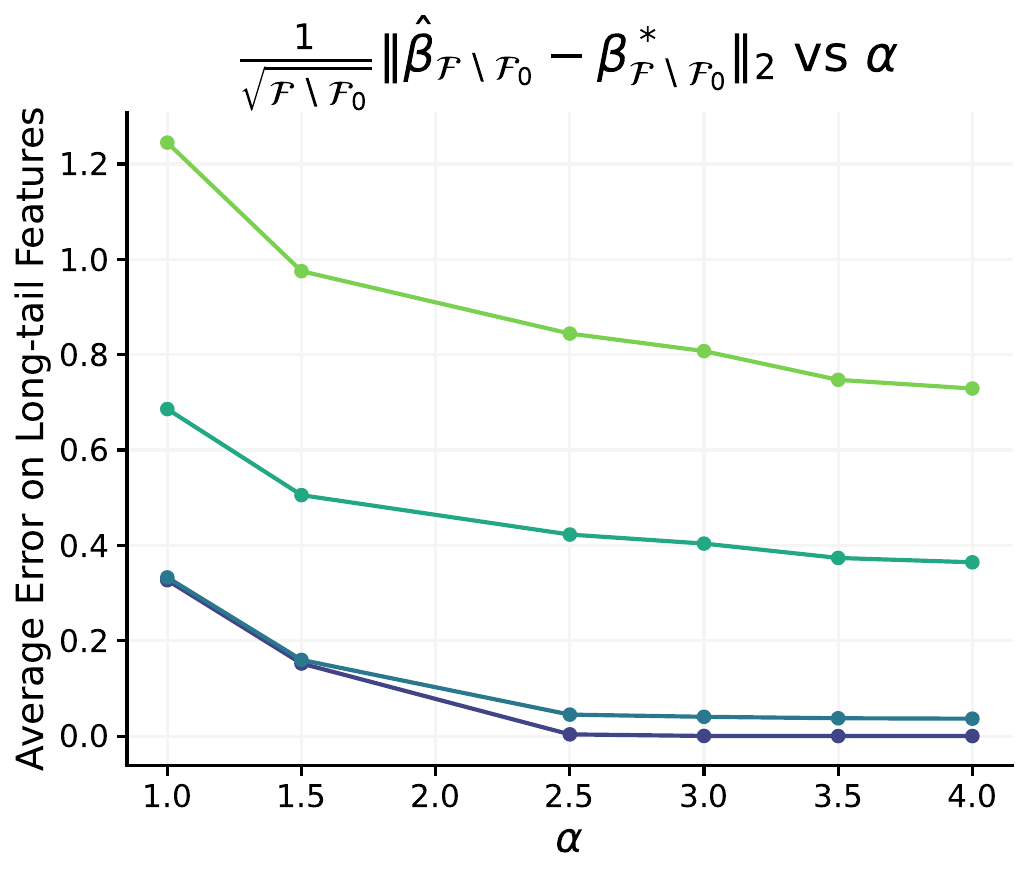}}
    \subcaptionbox{In-of-Distribution Test Loss\label{fig:iod_test_loss}}[0.24\textwidth]{%
        \includegraphics[width=\linewidth]{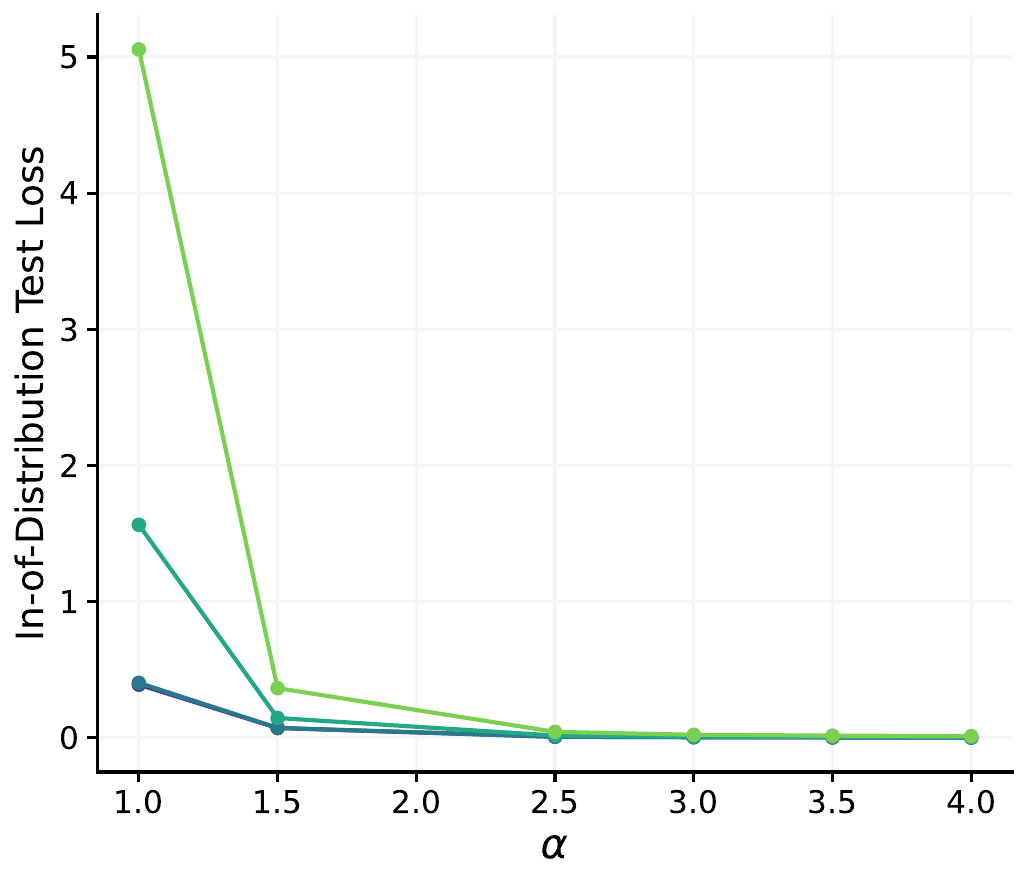}}
    \subcaptionbox{Out-of-Distribution Test Loss\label{fig:ood_test_loss}}[0.24\textwidth]{%
        \includegraphics[width=1.05\linewidth,trim={0 0 0 10},clip]{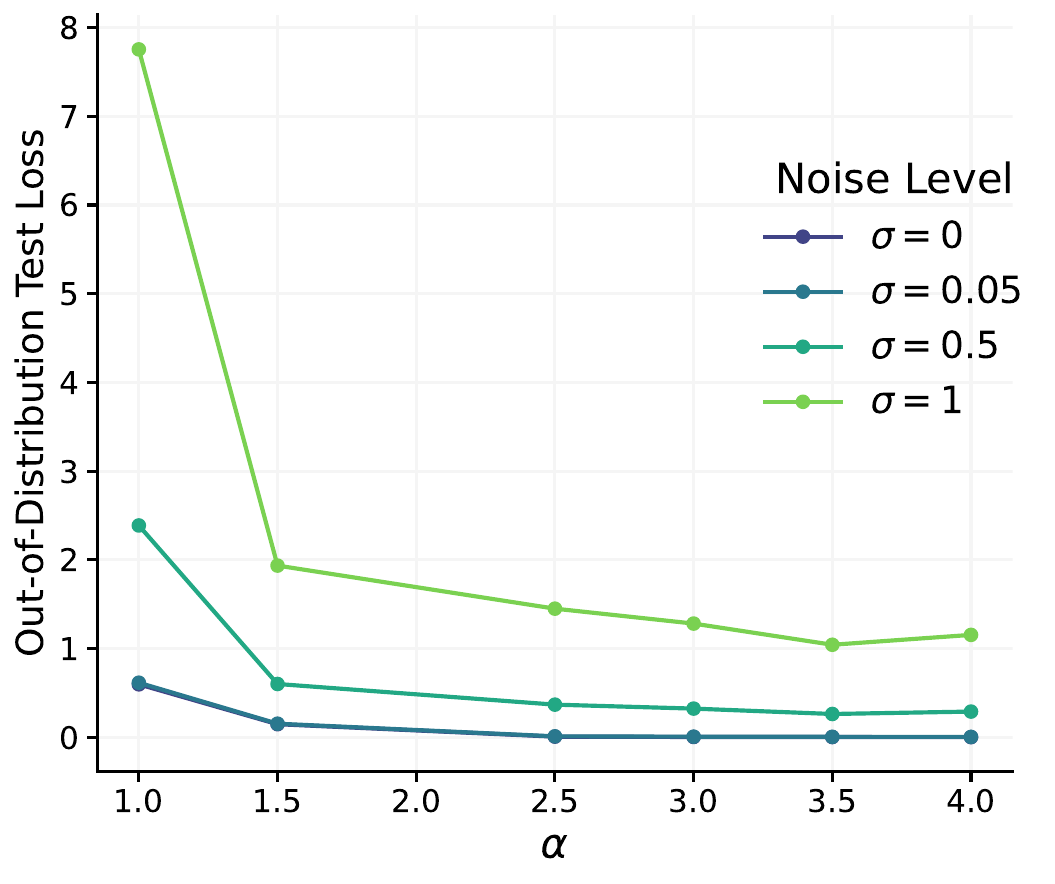}}
    \caption{(a)(b) Average Error on Common and Long-tail Features on Linear Model; (c)(d) In- and Out-of-Distribution Test Loss on Linear Model}
    \vspace{-7pt}
\end{figure}

\subsection{Task that Require Simple Composition}\label{sec: resnet mnist}
Next, we consider a simple task that would benefit from composition.  We constructed a new dataset based on MNIST. Each sample in this dataset consists of three randomly selected MNIST digit images in MNIST \textit{training} set, stacked along the channel dimension. The task is to predict the sum of 3 digits. Specifically, given that original MNIST images are of size $28 \times 28$, our new samples take the form of $3 \times 28 \times 28$, where each "channel" corresponds to a different digit. The digit selection also follows Zipf’s law distribution, meaning that lower digits (e.g., ‘0’) appear more frequently, while higher digits (e.g., ‘9’) appear less frequently.

To investigate the composition behavior and OOD generalization, every test sample is required to contain at least one digit ‘9’, introducing a distribution shift between the training and test sets. All test samples are drawn from the MNIST \textit{test} set. Importantly, this ensures that the digit ‘9’s used at test time have never been seen by the model during training.

\vspace{-3pt}
\paragraph{Model Architecture}
We consider several models using ResNet-18 as our backbone. The first few models process each channel separately through a ResNet-18 network. That is, given a sample with three stacked digit images, we compute the feature representation $g_i$ (second to last layer) and label ($f_i$) for each channel independently using the standard ResNet-18 forward propagation. The final model output is obtained by some aggregation over the individual channel outputs. For {\bf Sum}, the final output is simply the sum of $f_i$'s; for {\bf Linear}, the final output is a trained linear layer on top of the second-to-last layer $g_i$'s (where the 3 features are concatenated); for {\bf 2-layer}, the final output is a trained 2-layer network on top of the second-to-last layer features $g_i$'s. Both the ResNet-18 part and the aggregation layers are trained simultaneously.

We also consider {\bf Cross-channel}, which is a ResNet-18 model that is directly applied to the stacked images as if the channels are color channels (therefore the information between different images are mixed up starting from layer 1). 

The training set has 32,000 samples, with number 0 appear 77,435 times and number 8 appear 955 times. We change the number of times 9 appears to study the relationship between that and the test performance. For the test set, we enforce that each sample has at least one position equal to number 9, where the two other positions are sampled uniformly from 0-4. Note that with the number of samples we have in the training data, with high probability most number combinations in the test set will not appear in the training set.

Our results Figure ~\ref{fig:comparison_resnet_variants} show that {\bf Sum, Linear, 2-layer} models all perform well on the test distribution. Since the digit ‘9’ is rare in the training data and therefore not learned well by the models, this good performance suggests that these models exhibit some form of compositional behavior by leveraging memorized instances of ‘9’. On the other hand, Figure~\ref{fig:resnet_normal_main} shows that {\bf Cross-channel} model performs significantly worse on the test distribution, indicating that compositional capability still depends on the model architecture.

\begin{figure}
    \centering
    \begin{subfigure}[t]{0.45\textwidth}
        \centering
        \includegraphics[width=0.9\linewidth]{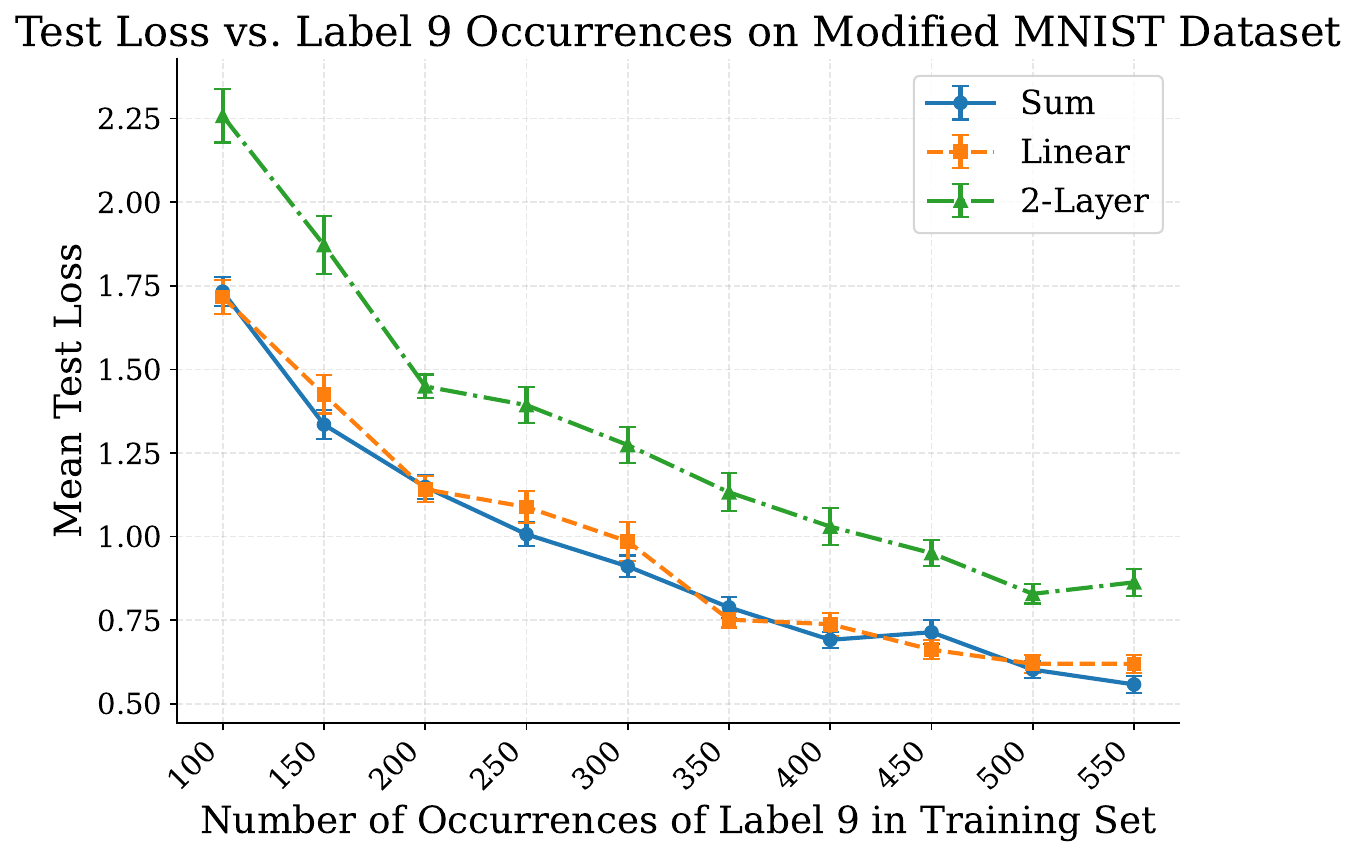}
        \caption{Test Loss of Sum, Linear, and 2-Layer Models}
        \label{fig:comparison_resnet_variants}
    \end{subfigure}
    \begin{subfigure}[t]{0.45\textwidth}
        \centering
        \includegraphics[width=0.9\linewidth]{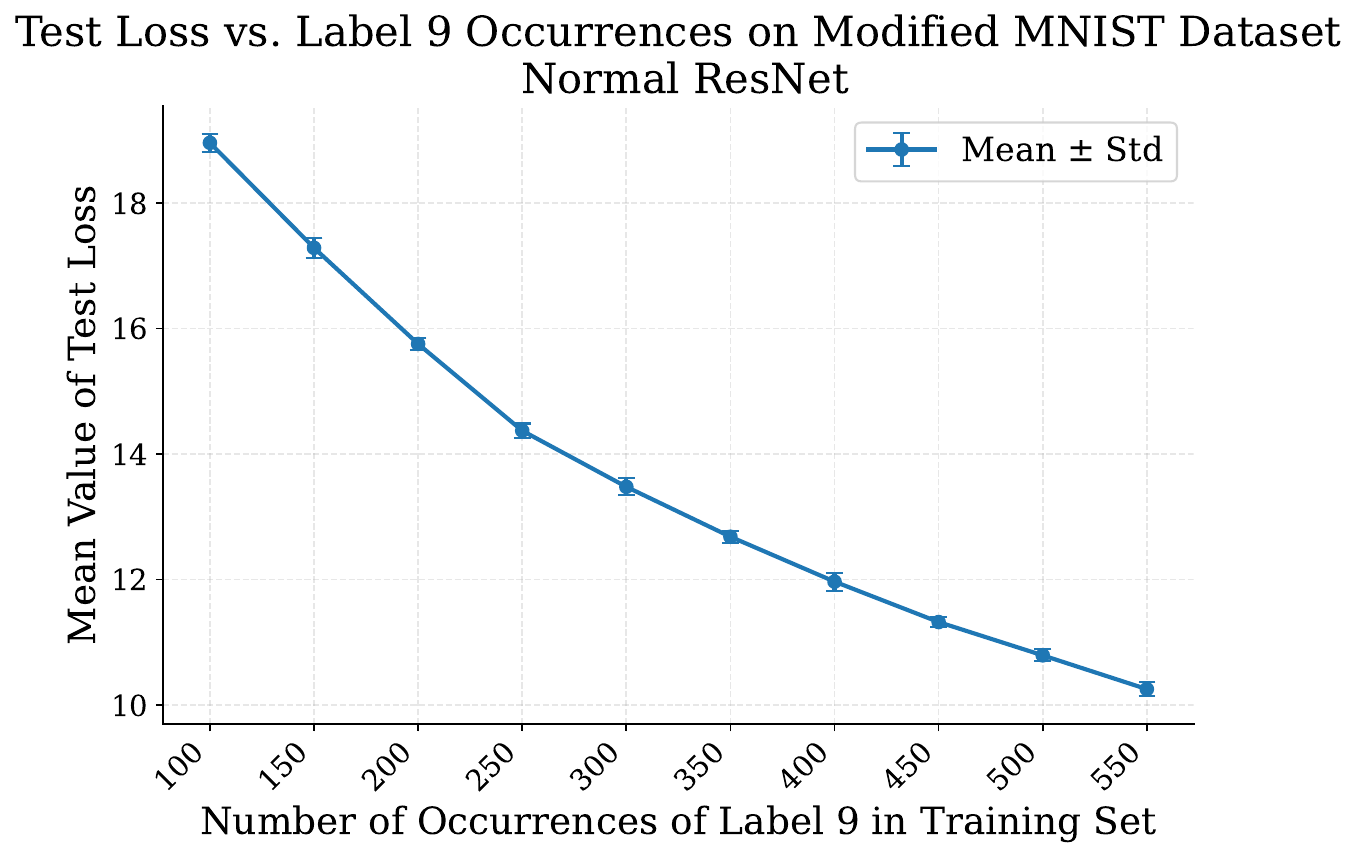}
        \caption{Test Loss of Cross-channel Model}
        \label{fig:resnet_normal_main}
    \end{subfigure}
    \caption{Test Performance Comparison Across {\bf Sum, Linear, 2-layer} and {\bf Cross-channel} Models}
    \label{fig:realworld_test_loss}
    \vspace{-12pt}
\end{figure}

\vspace{-3pt}
\subsection{Effect of Memorization}\label{sec: omniglot}
\vspace{-3pt}

To understand the effect of memorization, we further introduce even rarer examples. We consider a similar setting as the previous subsection. Except that in the training data, we choose 10 classes from Omniglot \citep{Lake2015HumanlevelCL} and select 1 image from each class. That image is assigned a label of 0-9. Each Omniglot image only appears exactly once in the training data (as one of the 3 images in one of the training data). The task is still to predict the sum of labels for the 3 images. 

For the test data, we consider images which contain two of the images from Omniglot (for the images that have appeared in the training data). Such test examples never appear in the training data and would not be close to any of them in Euclidean or embedding distance.

\begin{table}[ht]
\centering
\caption{Memorization behavior of ResNet variants under different weight decay (WD) settings}
\label{tab:memorization_omniglot}
\renewcommand{\arraystretch}{1.2}
\setlength{\tabcolsep}{12pt}
\begin{tabular}{lcc}
\toprule
\textbf{Model} & \textbf{WD = 0 (Memorization)} & \textbf{WD = 0.5 (No Memorization)} \\
\midrule
Sum          
    & \makecell{Test/Train Loss: 0.1760/0.0004} 
    & \makecell{Test/Train Loss: 2.8744/0.0245} \\
    \midrule
Linear    
    & \makecell{Test/Train Loss: 0.1662/0.0004} 
    & \makecell{Test/Train Loss: 1.5539/0.0505} \\
    \midrule
2-layer 
    & \makecell{Test/Train Loss: 0.1351/0.0010} 
    & \makecell{Test/Train Loss: 0.6447/0.0264} \\
    \midrule
Cross-channel
    & \makecell{Test/Train Loss:  22.9671/0.0023}
    & \makecell{Test/Train Loss:  23.6506/0.0375} \\
\bottomrule
\end{tabular}
\end{table}

As in Table~\ref{tab:memorization_omniglot}, despite only having seen the data once and the test data being significantly different, for {\bf Sum, Linear, 2-layer} models the performance on the test set is still reasonable when we don't use any regularization and make sure the model memorizes the training data. On the other hand, with weight decay, even though the training loss only increased slightly, the test performance become significantly worse. This suggests that memorization enables the model to develop a form of compositional ability that leverages the memorized Omniglot data, resulting in better generalization for memorizing models compared to those that do not memorize.
\textbf{Cross-channel} model does not perform well in this setting, further supporting our findings that the composition ability depends on architecture.

\section{Conclusions}
In this paper, we showed that memorization can contribute more to generalization and out-of-distribution generalization when the model has composition capability. Our theory builds on a simple linear setting and introduces new techniques for handling long-tail features. Empirically, we show that similar intuition extends to nonlinear settings, and whether models have composition capability partly depend on their architecture. The main limitation of our work is that the theoretical result is restricted to linear setting, while the experiments are on small networks for supervised learning problems. Our focus is on the positive role of memorization in enabling composition and illustrating when this occurs. There are also cases where memorization can hurt composition, for example when long-tail features are spuriously correlated with other features or labels, or when they are mislabeled. We hope this work serves as a starting point for understanding memorization and composition in more complicated models that go beyond supervised learning.

\section*{Acknowledgement}
This work is supported by NSF Award DMS-2031849. MZ also acknowledges the support of NSF TRIPODS II DMS 2023166.

\clearpage
\bibliographystyle{iclr2026_conference}
\bibliography{ref}

\newpage
\appendix

\section{Omitted Proof for noiseless results}
In this section we give the detailed proof for the noiseless results Theorem~\ref{thm: test loss} and Theorem~\ref{thm: ood}.

\thmtestloss*
\thmood*

\subsection{Proof overview}\label{appendix: noiseless proof overview}
We first show the structure of data $\mX$ as below: most of data has only common features; for the data that contains long-tail features, most of them has only one long-tail feature. The proof is based on column/row-wise concentration to exploit the combinatorial structure of $\mX$ instead of the concentration of whole matrix, which is difficult to handle due to high variance.
\begin{restatable}[Structure of $\mX$]{lemma}{lemstructureofX}\label{lem: structure of X}
    Suppose data generated from \eqref{eq: data generation}, with probability at least 0.99, the following hold with $k$ satisfying $np_k = \Theta(\ln^2 d)$ and $\pgtk\le 1-c_p$ with any constant $c_p$:
    \begin{enumerate}
        \item The number of samples whose non-zero entries are only within first $k$ entries is given by
        \[
            |S_0|:=\left|\{i\in[n]: (\vx_i)_j = 0,\forall j > k\}\right| \ge n(1-\pgtk) - O(\sqrt{n(1-\pgtk)\ln d}).
        \]
        Moreover, the dimension spans by these samples in $S_0$ is $k$.
        
        \item The number of samples that only have exactly one non-zero entry in long-tail entries (not in the first $k$ entries) is given by
        \begin{align*}
            |S_1|:=&\left|\{i\in[n]: \exists \ell >k \text{ such that } (\vx_i)_\ell\ne 0 \text{ and } (\vx_i)_j = 0,\forall j > k, j\ne \ell\}\right|\\ 
            \ge& n\pgtk(1-\pgtk)-O(\sqrt{n\pgtk(1-\pgtk)\ln d})
        \end{align*}
        and $|S_1|\lesssim n\pgtk$.

        \item Denote the set of features show up in data as 
        \[
            \gF = \{j\in[d]: \exists i\in[n] s.t. (\vx_i)_j\ne 0\}.
        \]
        Similarly define $\gF_0$ and $\gF_1$ for $S_0$ and $S_1$ as the set of features show up in $S_0$ and $S_1$.
        
        We have $|\gF\setminus(\gF_0\cup\gF_1)|\lesssim n\pgtk^2$ and the number of features show up in $S_0$ and $S_1$ is at least 
        \[
            |\gF_0\cup\gF_1|\ge \left(1-\Theta\left(\pgtk\ln^2 d\right)\right)|\gF|
        \] 
    \end{enumerate}
\end{restatable}

Given the structure of the data $\mX$, we obtain the following parameter recovery guarantee. Intuitively, since most of the data contains only the common feature, the common part is uniquely determined. Then, for the data points that contain only one long-tail feature, the long-tail part can also be uniquely recovered once common part is recovered. Together, these observations lead to the parameter recovery result presented below.
\begin{restatable}[Recovery of $\beta^*$]{lemma}{lemrecoveryofbeta}\label{lem: param recover}
    Suppose data generated from \eqref{eq: data generation} with noise $\sigma=0$, for $k$ that satisfies the condition of Lemma~\ref{lem: structure of X}, we have
    \begin{enumerate}
        \item Common features (Top-$k$ entries) are recovered: $\hvbetalek = \vbetalek^*$
        \item Most of long tail features shown up in data $X$(last-$d-k$ entries) are recovered:
        \[
            \frac{\left|\left\{i\in[d]: \hbeta_i = \beta^*_i\right\}\right|}{\left|\gF\right|}
            \ge 1-\Theta(\pgtk\ln^2 d).
        \]
        \item $\norm{\hvbeta}_2\le \norm{\vbeta^*}_2$ and $\hbeta_i=0$ for $i\not\in\gF$. As a result, $\norm{\hvbeta_{\gF\setminus(\gF_0\cup\gF_1)}}_2\le \norm{\vbeta^*_{\gF\setminus(\gF_0\cup\gF_1)}}_2$. 
    \end{enumerate}
\end{restatable}

As a corollary, we choose the threshold $k$ under our data generation process \eqref{eq: data power law} to get the following result.
\begin{corollary}\label{coro: param recover power law}
    Suppose data generated from \eqref{eq: data generation} with noise $\sigma=0$ and follows power law decay \eqref{eq: data power law} with $\alpha=1+c_\alpha$ for any constant $c_\alpha>0$ , then we have for $k=\Theta((ns/\ln^2 d)^{1/\alpha})$ (i.e. $np_k=\Theta(\ln^2 d)$)
    \begin{enumerate}
        \item Common features (Top-$k$ entries) are recovered: $\hvbetalek = \vbetalek^*$
        \item Most of long tail features shown up in data $X$ (last-$d-k$ entries) are recovered:
        \[
            \frac{\left|\left\{i\in[d]: \hbeta_i = \beta^*_i\right\}\right|}{\left|\gF\right|}
            \ge 1-
                \Theta\left(\frac{s^{\frac{1}{\alpha}}(\ln d)^{4-\frac{2}{\alpha}}}{n^{1-\frac{1}{\alpha}}}\right).
        \]
        \item $\norm{\hvbeta}_2\le \norm{\vbeta^*}_2$ and $\hbeta_i=0$ for $i\not\in\gF$. As a result, $\norm{\hvbeta_{\gF\setminus(\gF_0\cup\gF_1)}}_2\le \norm{\vbeta^*_{\gF\setminus(\gF_0\cup\gF_1)}}_2$.
    \end{enumerate}
\end{corollary}
\begin{proof}
    The proof follows from Lemma~\ref{lem: param recover} and Claim~\ref{claim: pk} by choosing $k=\Theta((ns/\ln^2 d)^{1/\alpha})$ (i.e., $np_k=\Theta(\ln^2 d)$).
\end{proof}

Now, given all parameter recovery results above, we can get the test loss and OOD loss guarantee. We give the details in the next section.

\subsection{Proof of main results}
Building on the results from the previous section, we are now ready to present the proofs of the main results. Both proofs mostly follow from the parameter recovery guarantee in Lemma~\ref{lem: param recover} and its specialization to power-law data in Corollary~\ref{coro: param recover power law}.
\thmtestloss*
\begin{proof}
    From Lemma~\ref{lem: param recover} we know
    \begin{align*}
        \E_{x\sim\gD}[(\hvbeta^\top\vx - \vbeta^{*\top}\vx)^2]
        =& \sum_i p_i (\hbeta_i-\beta_i^*)^2
        \lesssim \sum_{i>k} p_i (\beta_i^*)^2
            +p_k|\gF\setminus(\gF_0\cup\gF_1)|\\
        \lesssim& \pgtk
            + np_k\pgtk^2.
    \end{align*}
    When data follows power law decay \eqref{eq: data power law}, we know $k=\Theta((ns/\ln^2 d)^{1/\alpha})$ from $np_k=\Theta(\ln^2 d)$. Then from Claim~\ref{claim: pk} that $\pgtk\lesssim s\left(\frac{\ln^2 d}{ns}\right)^{1-\frac{1}{\alpha}}$ and get the result.
\end{proof}
\thmood*
\begin{proof}
    It suffices to let $\hat{\gF}=\gF_0\cup\gF_1$. From Lemma~\ref{lem: param recover} we know 
    \begin{align*}    
        \E_{x\sim\gD_{\widetilde{\gF}}}[(\hvbeta^\top\vx - \vbeta^{*\top}\vx)^2]
        =& \E_{x\sim\gD}[(\hvbeta^\top\vx - \vbeta^{*\top}\vx)^2]
        + \sum_{i\in\widetilde{\gF}} (1-p_i)(\hbeta_i-\beta_i^*)^2\\
        =& \E_{x\sim\gD}[(\hvbeta^\top\vx - \vbeta^{*\top}\vx)^2]
    \end{align*}
    where the calculation is same as in Theorem~\ref{thm: test loss}.
    
    When data follows power law decay \eqref{eq: data power law}, similar as in the proof of Theorem~\ref{thm: test loss} we know the bound. Size of $|\hat{\gF}|$ follows from Corollary~\ref{coro: param recover power law}.
\end{proof}

\subsection{Detailed proofs}
Here we give the omitted proofs in Appendix~\ref{appendix: noiseless proof overview}.
\lemstructureofX*
\begin{proof}
    We show one by one. 
    
    \paragraph{item 1}
    For any data $\vx\in\R^d$ generated by \eqref{eq: data generation}, we first have the probability
    \begin{align*}
        \P\left(x_i = 0, \forall i > k\right)
        = \prod_{i>k} (1-p_i)
        \ge 1 - \pgtk,
    \end{align*}
    where we use Lemma~\ref{lem: prod approx}.
    
    Since each data is generated i.i.d. following \eqref{eq: data generation}, by standard Chernoff's bound and union bound over $d$ features, we know with probability at least 0.999
    \[
        |S_0|=\left|\{i\in[n]: (\vx_i)_j = 0,\forall j > k\}\right| \ge n(1-\pgtk) - O(\sqrt{n(1-\pgtk)\ln d}).
    \]
    
    As for the dimension, since $np_k\gtrsim\ln^2 d$, we can show $\mX_{S_0,\le k}^\top \mX_{S_0,\le k}$ has full rank as in Lemma~\ref{lem:noisy-common-gram}, where $\mX_{S_0,\le k}$ is the data matrix $\mX\in\R^{n\times d}$ constrained with only rows in $S_0$ and first $k$ columns.

    \paragraph{item 2} Similar as item 1, for any data $\vx\in\R^d$ generated by \eqref{eq: data generation}, we first have the probability
    \begin{align*}
        &\P\left(\exists \ell >k \text{ such that } x_\ell\ne 0 \text{ and } x_j = 0,\forall j > k, j\ne \ell\right)\\
        =& \sum_{\ell>k}p_\ell\prod_{j>k,j\ne\ell}(1-p_j)
        \ge \sum_{\ell>k}p_\ell(1 - \pgtk+p_\ell)
        \ge \pgtk (1-\pgtk),
    \end{align*}
    where we use Lemma~\ref{lem: prod approx}. We can also see the above probability is at most $\pgtk$.

    Since each data is generated i.i.d. following \eqref{eq: data generation}, by standard Chernoff's bound and union bound over $d$ features, we know with probability at least 0.999
    \begin{align*}
        |S_1|
        =&\left|\{i\in[n]: \exists \ell >k \text{ such that } (\vx_i)_\ell\ne 0 \text{ and } (\vx_i)_j = 0,\forall j > k, j\ne \ell\}\right|\\
        \ge& n\pgtk(1-\pgtk) - O(\sqrt{n\pgtk(1-\pgtk)\ln d})
    \end{align*}
    and $|S_1|\lesssim n\pgtk$.

    \paragraph{item 3}

    For each given data $\vx$, let $Z(\vx)=\sum_{j>k}\ind_{x_j\ne 0}$ as the count of long tail features of data $\vx$. Thus, $|\gF\setminus(\gF_0\cup\gF_1)|\le \sum_{i=1}^n Z(\vx_i)\ind_{Z(\vx_i)\ge2}$. Since $\E_{\vx}[Z(\vx)\ind_{Z(\vx)\ge 2}]\le \E_{\vx}[Z(\vx)(Z(\vx)-1)]\le \pgtk^2$, by Markov bound we know w.p. 0.999 that $|\gF\setminus(\gF_0\cup\gF_1)|\lesssim n\pgtk^2$. 
    
    For every feature not in top-$k$, it at most shows up in $\Theta(np_i+\ln d)=O(\ln^2 d)$ data by Lemma~\ref{lem: matrix bernstein} and $np_i \le np_k =\Theta(\ln^2 d)$. Therefore, the total number of features show up in $S_0$ and $S_1$ is at least $|\gF_0\cup\gF_1|\ge k+|S_1|/O(\ln^2 d)\gtrsim k + (n\pgtk/\ln^2 d)$, where the error term $\sqrt{n\pgtk\ln d}$ in $|S_1|$ is absorbed by $k$. 
    Thus, using above $|\gF_0\cup\gF_1|\ge \Theta(n\pgtk/\ln^2 d)$ and $np_k=\ln^2 d$, we have 
    \begin{align*}    
        |\gF_0\cup\gF_1| 
        =& |\gF| \frac{1}{1+|\gF\setminus(\gF_0\cup\gF_1)|/|\gF_0\cup\gF_1|} 
        \ge |\gF|\frac{1}{1 + \Theta(\pgtk\ln^2 d)}\\
        \ge& |\gF| (1-\Theta(\pgtk\ln^2 d)).
    \end{align*}
\end{proof}

\lemrecoveryofbeta*
\begin{proof}
    We prove one-by-one. The results are mostly a combination of Lemma~\ref{lem: structure of X} and Claim~\ref{claim: pk}.

    \paragraph{item 1}
    For common features (top-$k$ entries), it suffices to only look at data in $S_0$, that is data only have non-zero entries in top-$k$ entries. From Lemma~\ref{lem: structure of X}, we know the only solution is $\hvbetalek = \vbetalek^*$.

    \paragraph{item 2}
    For long tail features, we only focus on those show up in $S_1$ (data only have 1 non-zero entries in tail entries). From item 1, we know top-$k$ entries are recovered. Thus, the only solution to data in $S_1$ is $\hbeta_i = \beta^*_i$ for $i\in\gF_1$.
    From Lemma~\ref{lem: structure of X}, we know $\gF_0\cup\gF_1$ covers most of the features show up in data: $|\gF_0\cup\gF_1|\ge (1-\Theta(\pgtk\ln^2 d))|\gF|$. Thus, from above we know the $\beta_i^*$ for $i\in\gF_0\cup\gF_1$ are recovered.

    \paragraph{item 3}
    This is directly implied by the meaning of min-$\ell_2$-norm interpolator.
\end{proof}

\section{Proofs for noisy case}\label{appendix: noisy}
In this section, we give the proofs for the noisy case results as below:
\thmtestlossnoisy*
\thmoodnoisy*

\subsection{Proof overview}\label{appendix: noisy proof overview}
Here we give the proof overview of noisy case results.

Besides the structure of data $\mX$ as described in Lemma~\ref{lem: structure of X}, we have the additional following structure under a stronger tail decay assumption. Now we can show all data has either 0 or 1 long tail feature.
\begin{restatable}{lemma}{lemstructureofXnoise}\label{lem: structure of X additional noise}
    Under Lemma~\ref{lem: structure of X} with additional requirement $n\pgtk^2 <c$ for a small enough constant $c$, then besides the conclusion of Lemma~\ref{lem: structure of X} we additionally have there is no combination of long tail features: 
    \[
        |S_{\ge 2}|=\left|\{i\in[n]: \exists \ell_1\ne\ell_2 >k \text{ such that } x_{\ell_1},x_{\ell_2}\ne 0\}\right| =0.
    \]
\end{restatable}

Given the structure of the data $\mX$, we can establish the following parameter recovery guarantee. Intuitively, the common features can be recovered accurately, since most of the data contains only common features (similar to the noiseless case). For the long-tail features, consider the simplest scenario where a feature appears only once in the entire dataset. In that case, the corresponding sample must have zero training loss, so the feature can be recovered up to the noise level $\sigma$. The general case follows a similar argument: we can control the training loss for samples that contain a given long-tail feature.
\begin{restatable}[Recovery of $\beta^*$]{lemma}{lemparamrecovernoisy}\label{lem: param recover noisy}
    Under Lemma~\ref{lem: structure of X} and Lemma~\ref{lem: structure of X additional noise}, with probability 0.99 we have 
    \begin{enumerate}
        \item Common features (Top-$k$ entries) are recovered: 
        \[
            \norm{\hvbetalek-\vbetalek^*}_{\mSigmalek}^2
            \lesssim \frac{\sigma^2k}{n}.
        \]
        \item All long tail features shown up in data $X$ (last-$d-k$ entries) are recovered upto noise level: 
        \[
            |\hbeta_j - \beta^*_j|^2
            \lesssim \frac{\sigma^2\ln d}{|S_{1,j}|}.
        \]
        where $S_{1,j}:=\{i\in S_1: (\vx_i)_j\ne 0\}$ is the set of data that feature $j$ and $S_1=\cup_{j>k}S_{1,j}$ forms a decomposition.
        \item $\hbeta_i=0$ for all $i\not\in (\gF_0\cup\gF_1)$.
    \end{enumerate}
\end{restatable}

As a corollary, we choose the threshold $k$ under our data generation process \eqref{eq: data power law} to get the following result.
\begin{corollary}\label{coro: param recover power law noisy}
    For data generating process \eqref{eq: data power law} of power law decay, with $\alpha=2+c_\alpha$ with any constant $c_\alpha>0$, we have for $k=\Theta((ns/\ln^2 d)^{1/\alpha})$ (i.e. $np_k=\Theta(\ln^2 d)$), with probability 0.99 same as Lemma~\ref{lem: param recover noisy}
    \begin{enumerate}
        \item Common features (Top-$k$ entries) are recovered: 
        \[
            \norm{\hvbetalek-\vbetalek^*}_{\mSigmalek}^2
            \lesssim \sigma^2\frac{s^{\frac{1}{\alpha}}}{n^{1-\frac{1}{\alpha}} \ln^{2/\alpha} d}.
        \]
        \item All long tail features shown up in data $X$ (last-$d-k$ entries) are recovered upto noise level: 
        \[
            |\hbeta_j - \beta^*_j|^2\lesssim \frac{\sigma^2\ln d}{|S_{1,j}|}.
        \]
        \item $\hbeta_i=0$ for all $i\not\in (\gF_0\cup\gF_1)$.
    \end{enumerate}
\end{corollary}
\begin{proof}
    The proof follows from Lemma~\ref{lem: param recover noisy} and Claim~\ref{claim: pk} by choosing $k=\Theta((ns/\ln^2 d)^{1/\alpha})$ (i.e., $np_k=\Theta(\ln^2 d)$).
\end{proof}

\subsection{Proof of main results}
Building on the results from the previous section, we are now ready to present the proofs of the main results. Both proofs mostly follow from the parameter recovery guarantee in Lemma~\ref{lem: param recover noisy} and its specialization to power-law data in Corollary~\ref{coro: param recover power law noisy}.
\thmtestlossnoisy*
\begin{proof}
    From Lemma~\ref{lem: param recover noisy} we know
    \begin{align*}    
        &\E_{x\sim\gD}[(\hvbeta^\top\vx - \vbeta^{*\top}\vx)^2]
        = \sum_i p_i (\hbeta_i-\beta_i^*)^2
        = \norm{\hvbetalek-\vbetalek^*}_{\mSigmalek}^2
            + \sum_{i>k}p_i(\hbeta_i-\beta_i^*)^2\\
        \lesssim& \frac{\sigma^2 k}{n}
            + \sum_{j>k,j\not\in\gF} p_j (\beta_j^*)^2
            +\sigma^2\sum_{j>k,j\in\gF}p_j\frac{\ln d}{|S_{1,j}|}.
    \end{align*}

    When data follows power law decay \eqref{eq: data power law}, we know $k=\Theta((ns/\ln^2 d)^{1/\alpha})$ from $np_k=\Theta(\ln^2 d)$. Then from Claim~\ref{claim: pk} that $\pgtk\lesssim s\left(\frac{\ln^2 d}{ns}\right)^{1-\frac{1}{\alpha}}$ and $|\beta_i^*|\lesssim 1$.
\end{proof}
\thmoodnoisy*
\begin{proof}
    From Lemma~\ref{lem: param recover noisy} we know
    \begin{align*}    
        &\E_{x\sim\gD_{\widetilde{\gF}}}[(\hvbeta^\top\vx - \vbeta^{*\top}\vx)^2]
        = \sum_{i\not\in\widetilde{\gF}} p_i (\hbeta_i-\beta_i^*)^2
        + \sum_{i\in\widetilde{\gF}} (\hbeta_i-\beta_i^*)^2\\
        \lesssim& 
        \norm{\hvbetalek-\vbetalek^*}_{\mSigmalek}^2
            + \sum_{j>k,j\not\in\widetilde{\gF}}p_j(\hbeta_j-\beta_j^*)^2
            + \sum_{j>k,j\in\widetilde{\gF}}(\hbeta_j-\beta_j^*)^2\\
        \lesssim& \frac{\sigma^2 k}{n}
            + \sum_{j>k,j\not\in\gF} p_j (\beta_j^*)^2
            +\sigma^2\sum_{j>k,j\in\gF\setminus\widetilde{\gF}}p_j\frac{\ln d}{|S_{1,j}|}
            +\sigma^2\sum_{j>k,j\in\widetilde{\gF}}\frac{\ln d}{|S_{1,j}|}.
    \end{align*}

    When data follows power law decay \eqref{eq: data power law}, we know $k=\Theta((ns/\ln^2 d)^{1/\alpha})$ from $np_k=\Theta(\ln^2 d)$. Then from Claim~\ref{claim: pk} we get the result.
\end{proof}

\subsection{Detailed proofs}
Here we present the detailed proofs in Appendix~\ref{appendix: noisy proof overview}
\lemstructureofXnoise*
\begin{proof}
    From the proof of Lemma~\ref{lem: structure of X} we know the probability of one data has at least 2 long tail features is
    \begin{align*}
        \P\left(\exists \ell_1\ne\ell_2 >k \text{ such that } x_{\ell_1},x_{\ell_2}\ne 0\right)
        \le \sum_{\ell_1,\ell_2>k}p_{\ell_1}p_{\ell_2}
        \lesssim \pgtk^2,
    \end{align*}
    Since each data is generated i.i.d. following \eqref{eq: data generation}, by Markov bound, we know with probability at least 0.999
    \[
        |S_{\ge 2}|=\left|\{i\in[n]: \exists \ell_1\ne\ell_2 >k \text{ such that } x_{\ell_1},x_{\ell_2}\ne 0\}\right| \lesssim n\pgtk^2<1,
    \]
    which implies $|S_{\ge 2}|=0$.
\end{proof}

\lemparamrecovernoisy*
\begin{proof}
    We show one by one.
    
    \paragraph{item 1 and 2.}
    We start with a close-form formula for $\hvbetalek$ and $\hbeta_j$ for $j>k$ by looking at the stationary condition of the least square.

    Given Lemma~\ref{lem: structure of X additional noise} we know $S_{\ge 2}=\varnothing$, so for each long tail feature it only shows up in $S_1$. Then, it is natural to further decompose $S_1=\cup_{j>k}S_{1,j}$, where $S_{1,j}$ is the set of data that feature $j$ and they are disjoint.

    For any $j>k$ and $j\in \gF$, since only samples with $x_j\ne 0$ depends on $\hbeta_j$, so the gradient w.r.t $\hbeta_j$ satisfies
    $0=\sum_{\vx\in S_{1,j}} x_j (\vxlek^\top(\hvbetalek-\vbetalek^*)+x_j(\hbeta_j-\beta_j^*)-\xi)$, which leads to
    \[
        \hbeta_j-\beta_j^*
        = \frac{1}{|S_{1,j}|}\mX_{:,j}^\top\vxi - \vv_j^\top(\hvbetalek-\vbetalek^*),
    \]
    where $\vv_j=\frac{1}{|S_{1,j}|}\mXlek^\top\mX_{:,j}$.
    Then consider the gradient w.r.t. $\hvbetalek$ similarly and plug in the above formula of $\hbeta_j-\beta_j^*$, we get
    \[
        \hvbetalek-\vbetalek^*
        =\mH^{-1}\mXlek^\top \mM\vxi,
    \]
    where 
    \[
        \mM=\mI-\sum_{j>k,j\in\gF}\frac{1}{|S_{1,j}|}\mX_{:,j}\mX_{:,j}^\top,\quad
        \mH = \mXlek^\top \mM \mXlek.
    \]

    Note that $\mM^2=\mM=\mM^\top$. For $\mH$ notice that for any vector $\vv$ we have $\vv^\top\mH\vv=\norm{\mM\mXlek\vv}_2^2\ge \norm{\mX_{S_0,\le k}\vv}_2^2$.  Thus, we have $\mH\succeq \mX_{S_0,\le k}^\top\mX_{S_0,\le k} \succeq \frac{n}{4} \mSigmalek$ where we use Lemma~\ref{lem:noisy-common-gram}. Together with $(|S_{1,j}|/n)\norm{\mSigmalek^{-1/2}\vv_j}_2^2\lesssim 1$ in Lemma~\ref{lem:noisy-signed-means}, we have
    \begin{align*}
        &\E_\vxi\left[\norm{\hvbetalek-\vbetalek^*}_{\mSigmalek}^2|\mX\right]
        =\sigma^2 \Tr(\mSigmalek\mH^{-1})
        \le \frac{4\sigma^2k}{n}\\
        &\E_\vxi\left[(\hbeta_j-\beta_j^*)^2|\mX\right]
        = \sigma^2\left(\frac{1}{|S_{1,j}|} + \vv_j^\top\mH^{-1}\vv_j\right)
        \le \sigma^2\left(\frac{1}{|S_{1,j}|} + \frac{4}{n}\norm{\mSigmalek^{-1/2}\vv_j}_2^2\right)
        \lesssim \frac{\sigma^2}{|S_{1,j}|}
    \end{align*}

    Then, notice that for every fixed $\mX$ satisfying the preceding good conditions, $\hvbetalek-\vbetalek^*$ and $\hbeta_j-\beta_j^*$ are all mean-zero Gaussians as they only depend linearly on noise $\vxi$. Therefore, by Lemma~\ref{lem:noisy-gaussian-quadratic} and union bound over $d$ features, we know with probability 0.999
    \begin{align*}
        &\norm{\hvbetalek-\vbetalek^*}_{\mSigmalek}^2
        \lesssim \frac{\sigma^2k}{n},\quad
        (\hbeta_j-\beta_j^*)^2
        \lesssim \frac{\sigma^2\ln d}{|S_{1,j}|}.
    \end{align*}

    In above, we exclude the bad event in Lemma~\ref{lem: structure of X additional noise}, \(|S_0|< n/2\), Lemma~\ref{lem:noisy-common-gram}, Lemma~\ref{lem:noisy-signed-means}, and Lemma~\ref{lem:noisy-gaussian-quadratic}. We could choose small enough fail probability, which supports the claimed \(0.99\) probability.

    \paragraph{item 3}
    From the discussion at the beginning of the proof, we know $|S_{\ge 2}|=0$, that is these feature does not show up in the data. Thus, we know the result is directly implied by the meaning of min-$\ell_2$-norm solution.
\end{proof}
\section{Technical lemma}
We collect few technical lemma that are used in the proof.
\begin{lemma}\label{lem: prod approx}
    If $p_i \le 1$ for $i> k$, then
    \[
        \prod_{i>k} (1-p_i) \ge 1 - \pgtk.
    \]
\end{lemma}
\begin{proof}
    We have
    \begin{align*}
        \prod_{i>k} (1-p_i)
        = 1 -\pgtk + \sum_{j=2}^d\underbrace{\sum_{k<i_1<i_2<\cdots<i_j\le d} (-1)^j p_{i_1}\cdots p_{i_j}}_{R_j}.
    \end{align*}

    When $j$ is even, first note that $R_j > 0$. Also,
    \begin{align*}
        R_j + R_{j+1}
        =& \sum_{k<i_1<i_2<\cdots<i_j\le d} (-1)^j p_{i_1}\cdots p_{i_j} 
            + \sum_{k<i_1<i_2<\cdots<i_{j+1}\le d} (-1)^{j+1} p_{i_1}\cdots p_{i_{j+1}}\\
        =& \sum_{k<i_1<i_2<\cdots<i_j\le d} (-1)^j p_{i_1}\cdots p_{i_j} (1 - \sum_{i_j<i_{j+1}\le d} p_{i_{j+1}})\\
        \ge& 0,
    \end{align*}
    as long as $\pgtk \le 1$.
    
    Therefore, we know $\sum_{j=2}^d R_j \ge 0$ when $\pgtk \le 1$. This implies
    \begin{align*}
        \prod_{i>k} (1-p_i)
        \ge 1 -\pgtk.
    \end{align*}

    When $\pgtk >1$, since $p_i<1$ for $i>k$, so 
    \begin{align*}
        \prod_{i>k} (1-p_i)
        \ge 0
        \ge 1 -\pgtk.
    \end{align*}
\end{proof}

The lemma below shows the data matrix $\mXlek^\top\mXlek$ after normalization is close to identity. The truncation is only to apply Matrix Berstein inequality.
\begin{lemma}
\label{lem:noisy-truncated-rows}
    Consider $\mX\in\R^{N\times d}$ following the data generating process \eqref{eq: data generation} with $N$ data. Fix $1\le k\le d$ and $L\ge1$. Define
\[
 \widetilde{\mathbf x}_{\le k}
 =\boldsymbol\Sigma_{\le k}^{-1/2}\mathbf x_{\le k},
 \qquad R=8\left(k+\frac{L}{p_k}\right),
 \qquad
 \widetilde{\mathbf x}^{\,\mathrm{tr}}_{\le k}
 =\widetilde{\mathbf x}_{\le k}
 \mathbf1\{\|\widetilde{\mathbf x}_{\le k}\|_2^2\le R\},
\]
where $\boldsymbol\Sigma_{\le k}=\operatorname{diag}(p_1,\ldots,p_k)$.
Then
\begin{align*}
 \P\bigl(\|\widetilde{\mathbf x}_{\le k}\|_2^2>R\bigr)
 \le e^{-2L},\quad
 \P\bigl(\|\widetilde{\mathbf x}_{\le k}\|_2^2>R\mid x_j\ne0\bigr)
 \le e^{-2L}\qquad (j\le k).
\end{align*}
The truncated vector satisfies
\[
 \mathbb E\widetilde{\mathbf x}^{\,\mathrm{tr}}_{\le k}=0,
 \qquad \|\widetilde{\mathbf x}^{\,\mathrm{tr}}_{\le k}\|_2^2\le R,
 \qquad
 (1-e^{-2L})I_k
 \preceq
 \mathbb E\left[
 \widetilde{\mathbf x}^{\,\mathrm{tr}}_{\le k}
 \widetilde{\mathbf x}^{\,\mathrm{tr}\top}_{\le k}
 \right]
 \preceq I_k.
\]
In particular, its expected squared norm is at most $k$.
\end{lemma}

\begin{proof}
The random variable $\|\widetilde{\mathbf x}_{\le k}\|_2^2
 =\sum_{\ell\le k}\frac{x_\ell^2}{p_\ell}$
has mean $k$. Its variance is
$\sum_{\ell\le k}(1-p_\ell)/p_\ell\le k/p_k$, and each centered summand has absolute value at most $1/p_k$. The scalar version of Lemma~\ref{lem: matrix bernstein} gives
\[
 \mathbb P\left(
 \|\widetilde{\mathbf x}_{\le k}\|_2^2
 > k+\sqrt{\frac{6kL}{p_k}}+\frac{2L}{p_k}
 \right)\le 2e^{-3L}\le e^{-2L}.
\]
The displayed threshold is at most $R$, which proves the first probability bound.

Conditional on $x_j\ne0$, the same squared norm is
\[
 \frac1{p_j}
 +\sum_{\substack{\ell\le k\\\ell\ne j}}\frac{x_\ell^2}{p_\ell}.
\]
The other coordinates remain their original independent laws. The conditional mean is at most $k+1/p_k$, the variance is at most $k/p_k$, and each random centered summand is bounded by $1/p_k$. Apply the same Lemma~\ref{lem: matrix bernstein}. Since $L\ge1$ and $\sqrt{6kL/p_k}\le2k+L/p_k$, its conditional threshold is at most
\[
 k+\frac1{p_k}+\sqrt{\frac{6kL}{p_k}}+\frac{2L}{p_k}
 \le 3k+\frac{4L}{p_k}\le R.
\]
The failure probability is again at most $2e^{-3L}\le e^{-2L}$.

The truncation event depends only on the coordinate magnitudes. Coordinate sign symmetry therefore makes the truncated mean zero and its covariance diagonal. Its $j$th diagonal entry equals
\begin{align*}
 &\E\left[
 \frac{x_j^2}{p_j}\ind\{\|\widetilde{\vx}_{\le k}\|_2^2\le R\}
 \right]
 =
 \P\bigl(\|\widetilde{\vx}_{\le k}\|_2^2\le R\mid x_j\ne0\bigr)
 \in [1-e^{-2L},1].
\end{align*}
This proves the covariance bounds. The remaining follow from the definition of the truncation.
\end{proof}

We now lower bound the common-feature Gram matrix.
\begin{lemma}
\label{lem:noisy-common-gram}
Consider $\mX\in\R^{N\times d}$ following the data generating process \eqref{eq: data generation} with $N$ data. Fix $1\le k\le d$ and $0<\delta<1/4$. If $kp_k\lesssim 1$, $Np_k\gtrsim \ln^2(Nd/\delta)$,
then, with probability at least $1-\delta$,
\[
 \mXlek^{\top}\mXlek
 \succeq\frac N2\mSigmalek.
\]
In particular, $\mathbf X_{\le k}$ has rank $k$.

As a result, condition on any fixed $\mathbf X_{>k}$ with $|S_0|\ge n/2$. If $kp_k\lesssim 1$ and $np_k\gtrsim\ln^2(nd/\delta)$,
then, with conditional probability at least $1-\delta$ over $\mathbf X_{\le k}$,
\[
 \mathbf X_{S_0,\le k}^{\top}\mathbf X_{S_0,\le k}
 \succeq\frac{|S_0|}{2}\boldsymbol\Sigma_{\le k}
 \succeq\frac n4\boldsymbol\Sigma_{\le k}.
\]
\end{lemma}

\begin{proof}
Set $L=\ln(128Nd/\delta)$ and $R= 8(k+(L/p_k))$. Apply Lemma~\ref{lem:noisy-truncated-rows} separately to the normalized rows. Write
\[
 Y_i=
 \widetilde{\mathbf x}^{\,\mathrm{tr}}_{i,\le k}
 \widetilde{\mathbf x}^{\,\mathrm{tr}\top}_{i,\le k},
 \qquad \mu=\mathbb E Y_i.
\]
The matrices $Y_i$ are independent, and
\[
 0\preceq Y_i\preceq RI_k,
 \qquad \tfrac34I_k\preceq\mu\preceq I_k,
 \qquad \|Y_i-\mu\|_{\mathrm{op}}\le R.
\]
Moreover,
\[
 \mathbb E(Y_i-\mu)^2
 =\mathbb E Y_i^2-\mu^2
 \preceq R\mu\preceq RI_k.
\]
Apply Lemma~\ref{lem: matrix bernstein} to $S_i=Y_i$. The dimensions are $k\times k$, the summand deviation bound is $R$, and the matrix variance is at most $NR$. At deviation $N/4$, this gives
\[
 \mathbb P\left(
 \left\|\sum_{i=1}^N(Y_i-\mu)\right\|_{\mathrm{op}}>N/4
 \right)
 \le2k\exp\left(-\frac{3N}{104R}\right)
\]
Since $R\lesssim\frac{L}{p_k}$, the assumption on $Np_k$ makes this probability at most $\delta$ when hidden constant is sufficiently large. Outside this exceptional event,
\[
 \sum_{i=1}^N Y_i\succeq\frac N2I_k.
\]
Every original row outer product dominates its truncated outer product. Hence
\[
 \mSigmalek^{-1/2}
 \mX_{\le k}^{\top}\mX_{\le k}
 \mSigmalek^{-1/2}
 \succeq\sum_{i=1}^N Y_i\succeq\frac N2I_k.
\]
This proves the first claim.

For the conditional claim, $S_0$ is fixed by $\mX_{>k}$. The common-feature rows indexed by $S_0$ retain their original independent laws. Apply the first claim with $N=|S_0|$. Indeed,
\[
 |S_0|p_k\ge\frac{np_k}{2}
 \gtrsim \ln^2(|S_0|d/\delta).
\]
This verifies the required sample-size condition.
\end{proof}

Next we bound the cross interaction between common and tail feature.
\begin{lemma}
\label{lem:noisy-signed-means}
Consider $\mX\in\R^{n\times d}$ following the data generating process \eqref{eq: data generation} with $n$ data. Fix $1\le k<d$ and $0<\delta<1/4$, and let $L=\ln(128nd/\delta)$, $R=8\left(k+\frac{L}{p_k}\right)$.
Condition on any fixed realization of $\mathbf X_{>k}$. For each $j>k$ with $|S_{1,j}|>0$, define
\[
 \vv_j=
 \frac1{|S_{1,j}|}
 \sum_{i\in S_{1,j}}(\mathbf x_i)_j(\mathbf x_i)_{\le k}.
\]
Then with conditional probability at least $1-\delta$ over $\mathbf X_{\le k}$, simultaneously for every such $j$,
\[
 \norm{\mSigmalek^{-1/2}\vv_j}_2^2
 \lesssim
 \min\left\{
 R,\ \left(\frac{kL}{|S_{1,j}|}
       +\frac{RL^2}{|S_{1,j}|^2}\right)
 \right\},
\]
and consequently
\[
 \frac{|S_{1,j}|}{n}
 \|\boldsymbol\Sigma_{\le k}^{-1/2}\vv_j\|_2^2
 \lesssim \frac{kL}{n}+\frac{L^2}{np_k}.
\]
On this same event, every symmetric matrix $H$ satisfying
$H\succeq(n/4)\boldsymbol\Sigma_{\le k}$ obeys
\[
 \vv_j^\top \mH^{-1}\vv_j
 \lesssim\frac{1}{|S_{1,j}|}
 \left(\frac{kL}{n}+\frac{L^2}{np_k}\right).
\]
In particular, if $kp_k\lesssim 1$ and $np_k\ge L^2$, then
$v_j^\top H^{-1}v_j\lesssim 1/|S_{1,j}|$ simultaneously for all such $j$.
\end{lemma}

\begin{proof}
The tail matrix fixes the pairwise disjoint groups $S_{1,j}$ and their signs $(\mathbf x_i)_j\in\{-1,1\}$. It does not change the distribution of $\mathbf X_{\le k}$.

Apply Lemma~\ref{lem:noisy-truncated-rows} separately to the normalized common-feature rows. The truncated vectors remain independent. They have mean zero, norm at most $\sqrt R$, covariance at most $I_k$, and expected squared norm at most $k$.

For a fixed nonempty group, apply Lemma~\ref{lem: matrix bernstein} to the $k\times1$ matrices
\[
 S_i=(\mathbf x_i)_j\widetilde{\mathbf x}^{\,\mathrm{tr}}_{i,\le k},
 \qquad i\in S_{1,j}.
\]
Conditional on $\mathbf X_{>k}$, the summands are independent and centered, and their norms are at most $\sqrt R$. The two variance quantities in Lemma~\ref{lem: matrix bernstein} satisfy
\[
 \left\|\sum_{i\in S_{1,j}}\mathbb E[S_iS_i^\top\mid\mathbf X_{>k}]\right\|_2
 \le |S_{1,j}|,
 \qquad
 \sum_{i\in S_{1,j}}\mathbb E[S_i^\top S_i\mid\mathbf X_{>k}]
 \le |S_{1,j}|k.
\]
Thus Lemma~\ref{lem: matrix bernstein} gives
\[
 \left\|\sum_{i\in S_{1,j}}S_i\right\|_2
 \le\sqrt{2|S_{1,j}|kL}+\frac23\sqrt R\,L
\]
except with conditional probability at most $(k+1)e^{-L}$. Squaring yields
\[
 \left\|\sum_{i\in S_{1,j}}S_i\right\|_2^2
 \lesssim \bigl(|S_{1,j}|kL+RL^2\bigr).
\]

There are at most $n$ nonempty groups. Union over the groups and then intersect with the event that no row was truncated. The total failure probability is at most
\[
 n(k+1)e^{-L}+ne^{-2L}<\delta.
\]
Thus the preceding bound holds for the original rows in every group. On the same event, Jensen's inequality also gives
\[
 \left\|
 \frac1{|S_{1,j}|}\sum_{i\in S_{1,j}}
 (\mathbf x_i)_j\widetilde{\mathbf x}_{i,\le k}
 \right\|_2^2\le R.
\]
Dividing the Bernstein bound by $|S_{1,j}|^2$ proves the displayed minimum.

If $|S_{1,j}|\le L$, use the bound $R$ to obtain
\[
 \frac{|S_{1,j}|}{n}
 \|\boldsymbol\Sigma_{\le k}^{-1/2}v_j\|_2^2
 \le\frac{LR}{n}.
\]
If $|S_{1,j}|\ge L$, use the other bound to obtain
\[
 \frac{|S_{1,j}|}{n}
 \|\boldsymbol\Sigma_{\le k}^{-1/2}\vv_j\|_2^2
 \le C\left(\frac{kL}{n}
       +\frac{RL^2}{n|S_{1,j}|}\right)
 \le C\left(\frac{kL}{n}+\frac{RL}{n}\right).
\]
Substituting $R=8(k+L/p_k)$ proves the second assertion.

Finally, $\mH\succeq(n/4)\boldsymbol\Sigma_{\le k}$ implies
$H^{-1}\preceq(4/n)\boldsymbol\Sigma_{\le k}^{-1}$. Therefore
\[
 \vv_j^\top \mH^{-1}\vv_j
 \le\frac4n\|\boldsymbol\Sigma_{\le k}^{-1/2}\vv_j\|_2^2.
\]
The desired bound follows. 
\end{proof}

\subsection{Concentration inequality}
We collect few standard Matrix concentration inequalities here. The first one is the standard matrix Bernstein inequality.
\begin{lemma}[Matrix Bernstein Inequality (Corollary 6.1.2 in \citet{tropp2015introduction})]\label{lem: matrix bernstein}
    Consider a finite sequence $\{\mS_k\}$ of independent random matrices with common dimension $d_1 \times d_2$. 
    Assume that each matrix has uniformly bounded deviation from its mean:
    \[
    \norm{\mS_k - \E[\mS_k]}_2 \le L 
    \quad \text{for each index } k.
    \]
    Introduce the sum
    \[
    \mZ = \sum_k \mS_k,
    \]
    and let $\nu(\mZ)$ denote the matrix variance statistic for $\mZ$:
    \begin{align*}    
    \nu(\mZ) 
    =& \max\!\Bigl\{
        \norm{\E\bigl[(\mZ - \E[\mZ])(\mZ - \E[\mZ])^\top]}_2,\;
        \norm{\E\bigl[(\mZ - \E[\mZ])^\top(\mZ - \E[\mZ])\bigr]}_2
        \Bigr\}\\
    =& \max\!\left\{
        \norm{\sum_k \E\bigl[(\mS_k - \E[\mS_k])(\mS_k - \E[\mS_k])^\top\bigr]}_2,\;
        \norm{\sum_k \E\bigl[(\mS_k - \E[\mS_k])^\top(\mS_k - \E[\mS_k])\bigr]}_2
        \right\}.
    \end{align*}
    
    Then
    \[
        \E\norm{\mZ - \E[\mZ]}_2
        \le
        \sqrt{2\nu(\mZ)\ln\bigl(d_1 + d_2\bigr)}
        +
        \frac{1}{3} L \ln\bigl(d_1 + d_2\bigr).
    \]
    In particular, for all $t \ge 0$,
    \[
        \P\bigl\{\norm{\mZ - \E[\mZ]}_2\ge t\bigr\}
        \le
        (d_1 + d_2) \exp\Bigl(\frac{-t^2/2}{\nu(\mZ)+Lt/3}\Bigr).
    \]
    Consequently, for every $t>0$,
    \begin{align*}
     \P\left(
     \|\mZ-\mathbb E \mZ\|_2>
     \sqrt{2\nu(\mZ)t}+\frac23Lt
     \right)\le(d_1+d_2)e^{-t}.
     \label{eq:noisy-bernstein-inverted}
    \end{align*}
    If $\nu(\mZ)=0$, then $\mZ=\E \mZ$ almost surely and the conclusions are immediate.
\end{lemma}

The next one is bound on Gaussian vector's norm by its covariance.
\begin{lemma}[Gaussian squared norm (Lemma 1 in \cite{laurent2000adaptive})]
\label{lem:noisy-gaussian-quadratic}
Let $\mathbf g\in\mathbb R^r$ be a centered Gaussian vector with covariance $V\succeq0$. Then
\[
 \mathbb E\|\mathbf g\|_2^2=\operatorname{tr}V.
\]
For every $t>0$, with probability at least $1-e^{-t}$,
\begin{equation*}
 \|\mathbf g\|_2^2
 \le\operatorname{tr}V
 +2\sqrt{\operatorname{tr}(V^2)t}
 +2\|V\|_{\mathrm{op}}t.
\end{equation*}
On the same event, the following consequences hold:
\begin{align*}
 \|\mathbf g\|_2^2&\le5t\operatorname{tr}V
 &&\text{if }t\ge1;\label{eq:noisy-gaussian-trace}\\
 \|\mathbf g\|_2^2&\le a(r+2\sqrt{rt}+2t)\le3a(r+t)
 &&\text{if }V\preceq aI_r,\ a\ge0.
\end{align*}
The assertions also hold conditionally whenever the conditional law of $\mathbf g$ is centered Gaussian and $V$ is its conditional covariance.
\end{lemma}

Below is the standard Chernoff's bound for Bernoulli random variables.
\begin{lemma}[Chernoff Bounds]\label{lem: chernoff}
    Let \( X = \sum_{i=1}^{n} X_i \), where \( X_i = 1 \) with probability \( p_i \), and \( X_i = 0 \) with probability \( 1 - p_i \), and all \( X_i \) are independent. Let \( \mu = \mathbb{E}(X) = \sum_{i=1}^n p_i \). Then
    \begin{align*}
        \P(|X - \mu| \ge \delta\mu) \leq 2e^{-\mu \delta^2 / 3} \quad \text{for all } 0 < \delta < 1.
    \end{align*}
\end{lemma}

\section{Experiment Details and Additional Experiments}
In this section, we provide details of the experiments discussed in Section~\ref{sec: experiments}, along with a few additional experiments.
\begin{itemize}
    \item In Appendix~\ref{appendix:toy}, we describe experiments on a synthetic toy dataset.

    \item In Appendix~\ref{appendix:real}, we present details for Section~\ref{sec: resnet mnist}, which explores a modified version of MNIST with structured label imbalance, along with new results using a different target function.

    \item In Appendix~\ref{appendix:omniglot}, we provide details for Section~\ref{sec: omniglot}, where we augment an MNIST-based dataset with rarer samples from Omniglot and include an additional experiment.
    
    \item In Appendix~\ref{appendix:compute}, we report the compute we use for the experiments.

    \item In Appendix~\ref{appendix:mem score}, we show that long-tail data exhibit high memorization scores when measured by influence functions \citep{feldman2020neural}.
\end{itemize}

For experiments in this section, all datasets contain original labels ranging from 0 to 9. We deliberately control the number of samples per label in the training set to induce a long-tail distribution. More precisely, the number of samples for labels $i \in \{0, 1, \dots, 9\}$ follows Zipf’s law with exponent $s=2$. The unnormalized weight assigned to rank $i$ is defined as $w_i = (i+1)^{-s}$
and the final sampling probability for label $i$ is
$p_i = w_i/\sum_{j=0}^{9} w_j.$

Without loss of generality, we primarily focus on the samples with label \textbf{9} as the representative long-tail feature in the following experiments except for the experiment described in Appendix~\ref{appendix:omniglot}. We further analyze the performance of models that successfully memorize all long-tail features, with a particular emphasis on their generalization behavior on the test dataset.

To analyze the compositional effect, we focus on a particular long-tail feature and vary its frequency of appearance in the training set across 10 levels, ranging from 100 to 550 with increments of 50. For each occurrence level, we train $n$ independently initialized models and evaluate their performance on the test set. Specifically, we report the average test loss $\mu_{\text{test\_loss}}$ and the variance of this average $\text{Var}(\bar{x}) = \sigma^2/n$.

\subsection{Additional Synthetics Experiment: Gaussian mixtures}\label{appendix:toy}

\paragraph{Dataset Construction}\label{para:Toy_dataset}
We generate synthetic data based on multivariate normal distributions. Specifically, we randomly initialize 10 mean vectors, each corresponding to a digit from 0 to 9. Each mean vector is of dimension 100 and is sampled from a multivariate normal distribution with a mean of a zero vector and an identity covariance matrix. For each class $c \in \{0, 1, ..., 9\}$, samples are drawn from a multivariate normal distribution with the corresponding mean vector and a covariance matrix of $\sigma I$. The number of samples per class follows Zipf’s law decay distribution.

To construct the final dataset, we randomly sample triplets from the generated data, where each triplet consists of three individual samples. The label of a triplet is defined as the sum of the labels of its constituent samples. This results in a dataset of shape $n \times 300$.

\paragraph{Model Architecture}
We employ a two-layer neural network with an additive structure. During the forward pass, the model processes the input by dividing the 300-dimensional feature vector into three segments: dimensions 0–100, 100–200, and 200–300. The network independently computes the outputs for each segment, and the final prediction is obtained by summing the three outputs.

\paragraph{Training Setup}

We train the model using Mean Squared Error (MSE) loss and optimize it using the Adam optimizer. We use a learning rate of $\textbf{0.0001}$, train for $\textbf{300}$ epochs, with a batch size of $\textbf{256}$.

\paragraph{Evaluation Metric}\label{para:toy_metric}
To assess the compositional effect of memorizing long-tailed features on generalization, we visualize the results using a performance plot. The x-axis represents the total number of distinct ‘9’s the model has encountered during training, while the y-axis shows the corresponding test loss. For each x-axis value (ranging from 100 to 550), we train 20 independent models. Each data point is accompanied by an error bar, indicating the variance of mean of loss across test samples—smaller variance suggests higher confidence in the loss estimates.

Our test dataset is constructed as follows: we first generate 600 vectors for each class (0 through 9) by sampling from the same class-wise mean vectors used in the training dataset. To evaluate out-of-distribution generalization on the long-tail feature, we focus on class 9. For each vector labeled as 9, we randomly pair it with vectors from all other classes to create 2700 test samples. This results in a total of $600 \times 2700 = 1{,}620{,}000$ test examples. The goal of this setup is to examine whether memorization of the long-tail feature (class 9) generalizes to unseen vectors with the same label. If the model exhibits improved performance on these test examples, it suggests that long-tail memorization contributes positively to generalization.

\begin{figure}[th]
    \centering
    \includegraphics[width=0.5\linewidth]{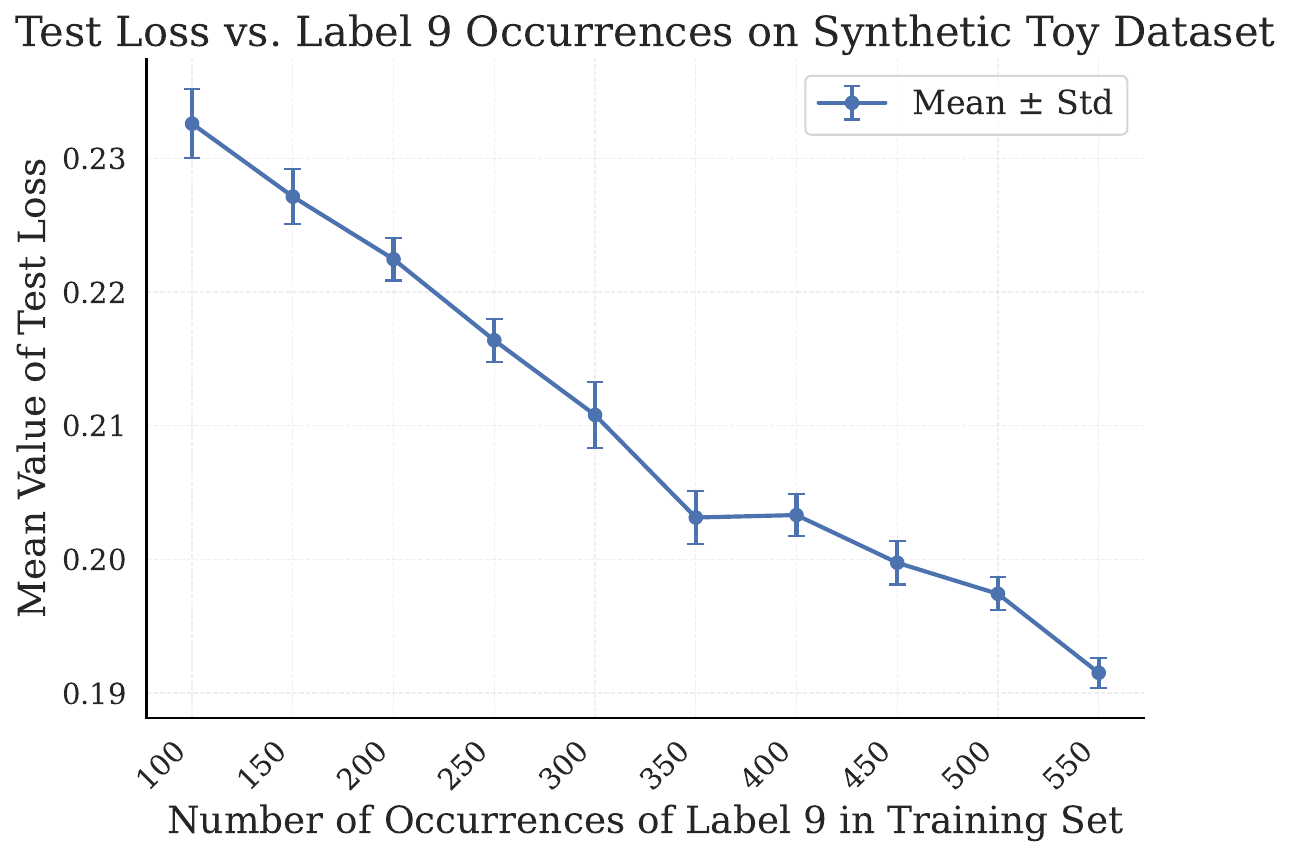} 
    \caption{Test loss as a function of the number of training occurrences of class 9.}
    \label{fig:longtail_loss}
\end{figure}

As shown in Figure~\ref{fig:longtail_loss}, the test loss on examples involving class 9 decreases as the number of class 9 samples in the training set increases. Notably, the class 9 vectors in the test set are not present in the training set, indicating that the observed performance gain reflects genuine generalization rather than rote memorization.

\subsection{Experiment Details in Section~\ref{sec: resnet mnist}: Long-tail MNIST data}\label{appendix:real}
We give more details of the experiments in Section~\ref{sec: resnet mnist}.
\paragraph{Dataset Construction}
In the training set, the number of samples per label is as follows: \{0: 77{,}435; 1: 19{,}358; 2: 8{,}603; 3: 4{,}839; 4: 3{,}097; 5: 2{,}150; 6: 1{,}580; 7: 1{,}209; 8: 955; 9: 774\}, resulting in a total of 120{,}000 samples.

It is worth noting that the original MNIST training set contains roughly 5{,}000--6{,}000 samples per digit. Therefore, for labels 0--3, we apply \textit{oversampling} to increase their frequency. All samples are then randomly shuffled and grouped into triplets, yielding a final training set of 40{,}000 triplets.

For the test set, we uniformly sample 200 samples for each class label from 0 to 4. Additionally, we sample 200 unseen digit-9 images (i.e., samples not included in the training set). Similar to the construction in the toy setting, test samples are generated by pairing one digit-9 sample with each of the $(200 \times 5)$ combinations from labels 0--4. In total, this results in $200 \times 200 \times 5 / 2 = 100{,}000$ test samples.

We clarify that all training samples are drawn exclusively from the MNIST \textit{training} set, and all test samples are drawn from the MNIST \textit{test} set.

\paragraph{Training Setup}

We train the model using Mean Squared Error (MSE) loss and optimize it using the Adam optimizer. We use a learning rate of $\textbf{0.0001}$, train for $\textbf{300}$ epochs, with a batch size of $\textbf{256}$, and no weight decay.

\paragraph{Evaluation Metric}\label{para:mnist_metric}

To evaluate how memorization of long-tailed features influences compositional generalization in a real-world setting, we adopt the same performance visualization approach as in our toy experiments.

For each x-axis value (ranging from 100 to 550), similarly, we train 100 independent models. Each model is trained on an 80\% random subset of the original dataset, ensuring that every subset contains the same number of ‘9’s. The reported test loss is the average over these 100 models.

Notably, since we control the number of ‘9’s in the training subsets, models trained with more ‘9’s are exposed to fewer instances of other digits. However, this does not significantly impact the observed downward trend in test loss. This further supports our hypothesis that memorizing long-tailed features substantially aids generalization.

\subsubsection{Additional experiment: Robustness Across Target Functions}\label{sec:robustness}

We further examine the robustness of the observed trends by modifying the target function used during training. Specifically, instead of using a simple sum of three labels as in Section~\ref{sec: resnet mnist}, we redefine the target as a weighted combination: 
\[
label = \sum_{i=1}^{3} i\times label_i
\]
This target is slightly more complex than the one in Section~\ref{sec: resnet mnist}.

In Figure~\ref{fig:resnet_newtarget}, we observe that the loss curves continue to exhibit a similar downward trend as the number of training instances from class 9 increases. This indicates our results in Section~\ref{sec: resnet mnist} is robust to different target.

Notably, the dataset and training setup are exactly the same as in the previous experiments. The only difference lies in the way labels are computed in the dataset.

\begin{figure}[ht] 
    \centering
    \begin{subfigure}[b]{0.45\textwidth}
        \centering
        \includegraphics[width=1\linewidth]{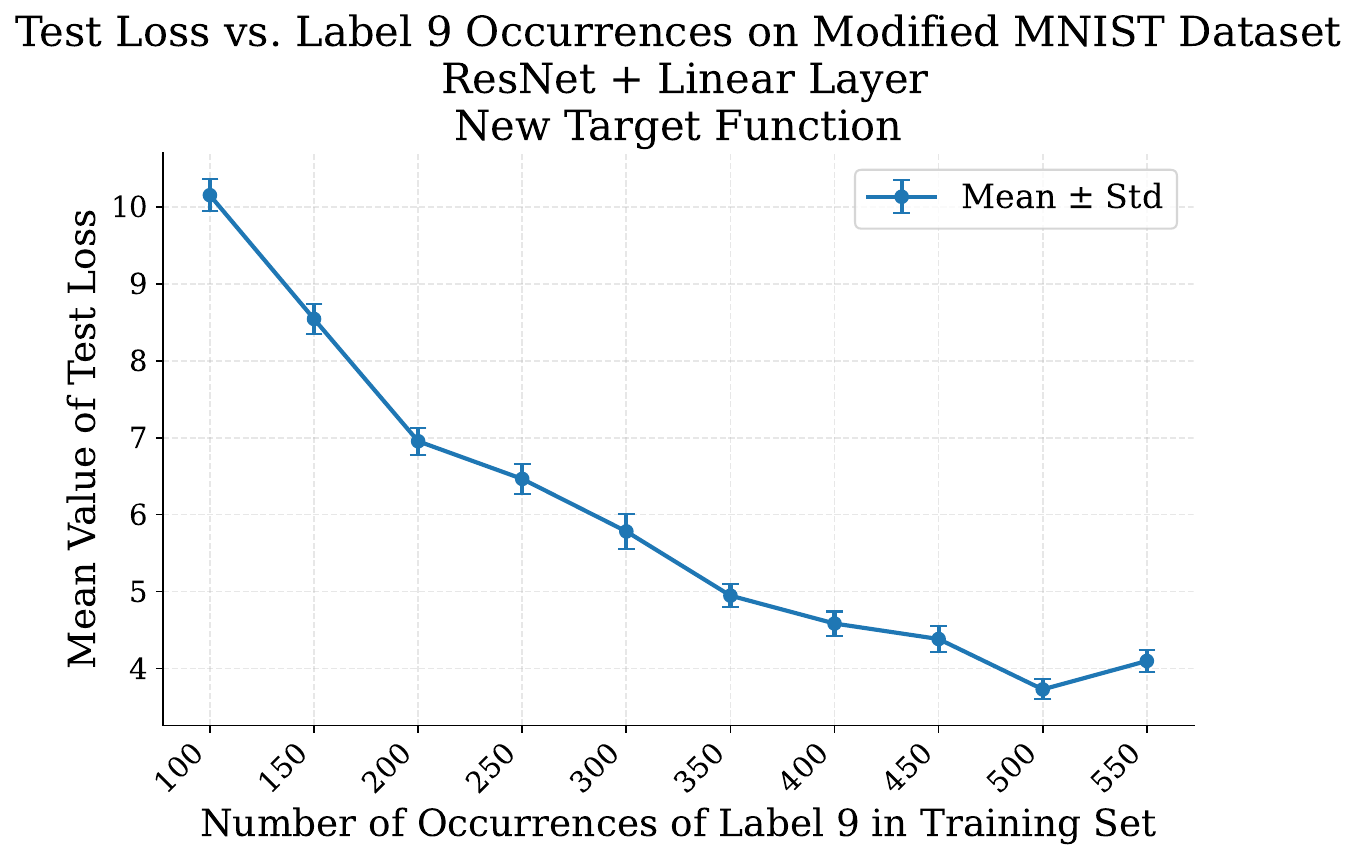} 
    \end{subfigure}
    \hfill
    \begin{subfigure}[b]{0.45\textwidth}
        \centering
        \includegraphics[width=1\linewidth]{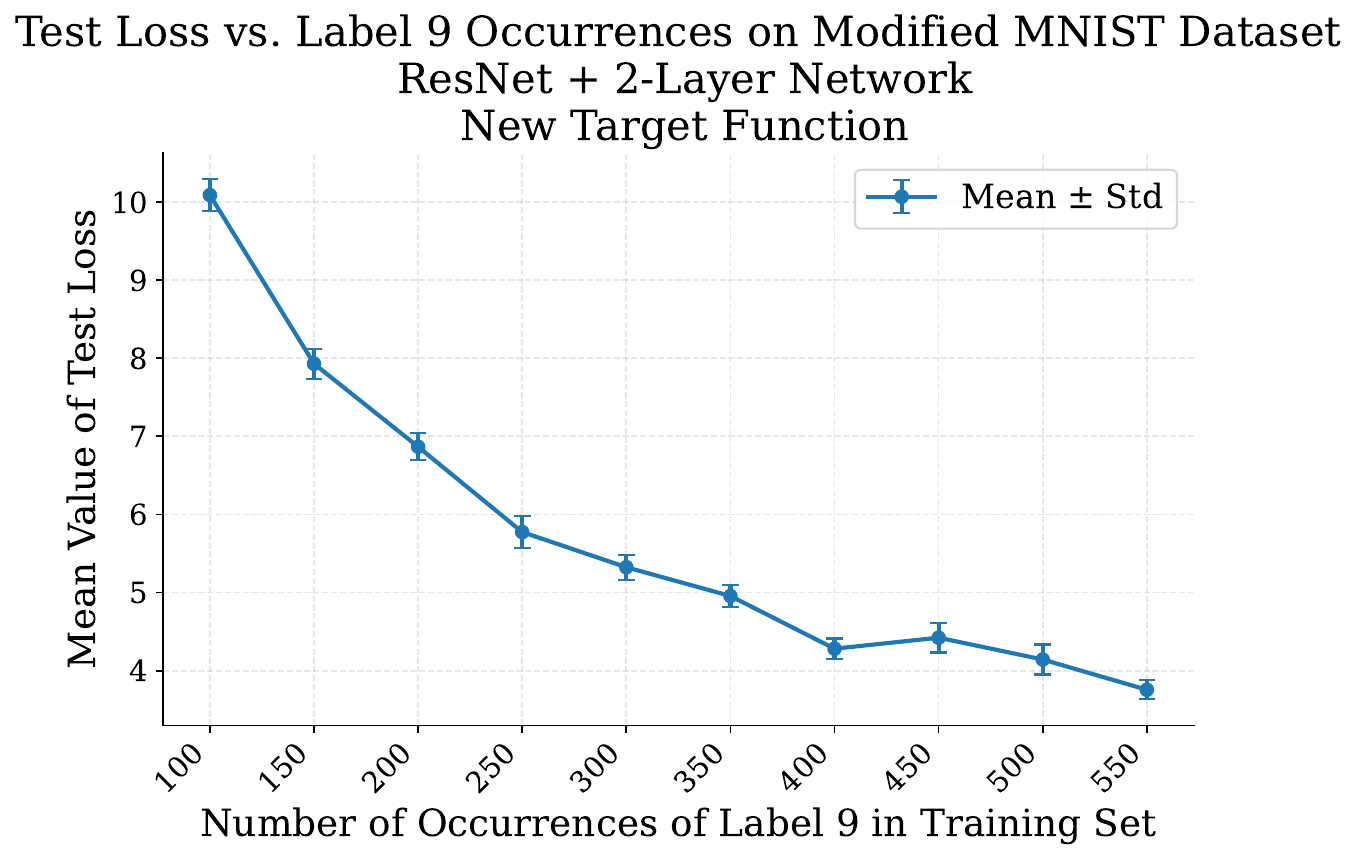} 
    \end{subfigure}
    \caption{\textbf{(Left)} ResNet+\textbf{Linear layer} with New Target Function: Test loss as a function of the number of training occurrences of class 9; \textbf{(Right)} ResNet+\textbf{2-layer network} with New Target Function: Test loss as a function of the number of training occurrences of class 9.
    }
    \label{fig:resnet_newtarget}
\end{figure}

\subsection{Experiment Details in Section~\ref{sec: omniglot}: Rarer Sample Case from Omniglot dataset}\label{appendix:omniglot}
Firstly, we give more details of the experiments in Section~\ref{sec: omniglot}.

\paragraph{Dataset Construction}
The MNIST part of the mixed dataset is constructed in the same way as the training set described in Appendix~\ref{appendix:real}. The only difference is that we additionally include samples from the Omniglot dataset. Specifically, we randomly select one sample from each class labeled 0--9 in the Omniglot training set, resulting in 10 Omniglot samples. These samples are preprocessed to match the format and size of the MNIST data and then mixed into the MNIST training set. To ensure the total number of samples is divisible by 3 (since samples are grouped into triplets), we randomly remove one MNIST sample after adding the Omniglot data. The final training set size is 40{,}003.

For the test set, we uniformly sample 100 MNIST samples from each class label 0--4, resulting in 500 MNIST test samples. We then consider all possible pairwise combinations among the 10 Omniglot samples, generating 45 unique Omniglot pairs. Each of these pairs is combined with the MNIST samples to form test triplets, resulting in a total of $45 \times 500 = 22{,}500$ test samples.

\paragraph{Training Setup}

We also train the model using Mean Squared Error (MSE) loss and optimize it using the Adam optimizer. We use a learning rate of $\textbf{0.0001}$, train for $\textbf{300}$ epochs, with a batch size of $\textbf{256}$.

\subsubsection{Additional experiment: More Omniglot data per class}
Similar as Section~\ref{sec: omniglot}, we also design experiments where Omniglot samples appear more frequently, and the test sample contains only one instance from Omniglot. Specifically, we introduce a variant of the 3-digit synthetic dataset by injecting distractor samples from the Omniglot dataset. To construct the training set, we randomly select 10 samples for each digit label (0–9) from Omniglot (totaling 100 samples) and mix them into the original training data. For the test set, we sample the same 10 Omniglot samples in training dataset per digit label (0–9). Additionally, we sample 200 test digits each from MNIST labels 0–4.

Each final test instance is formed by combining one digit from Omniglot and two digits from MNIST, resulting in test samples that include both known and out-of-distribution (OOD) components. In total, this produces $100 \times 5 \times 100 = 50,000$ test samples.

 Table~\ref{tab:memorization_transposed} show that the model still performs well even on test data that are not limited to long-tail feature combinations.

\begin{table}[ht]
\centering
\caption{Stronger Memorization Behavior of ResNet Variants Under Different Weight Decay Settings}
\vspace{5pt}
\label{tab:memorization_transposed}
\renewcommand{\arraystretch}{1.3}
\setlength{\tabcolsep}{10pt}

\begin{tabular}{lcc}
\toprule
\textbf{Model} & \textbf{Memorization (WD = 0)} & \textbf{No Memorization (WD = 0.5)} \\
\midrule
Sum          
    & \makecell{Test: 0.0697 \\ Train: 0.0002} 
    & \makecell{Test: 0.8266 \\ Train: 0.0538} \\
    \midrule
Linear    
    & \makecell{Test: 0.1171 \\ Train: 0.0006} 
    & \makecell{Test: 0.9571 \\ Train: 0.0715} \\
    \midrule
2-Layer 
    & \makecell{Test: 0.0650 \\ Train: 0.0013} 
    & \makecell{Test: 0.6579 \\ Train: 0.0196} \\
\bottomrule
\end{tabular}
\end{table}

\subsection{Compute Requirements}\label{appendix:compute}
The computational resources used for all our experiments are as follows: 10 GTX 2080 Ti GPUs with 60 GB of RAM.
For the Experiments on Linear Model, completing all runs takes approximately 3 hours.
For the Simple Task with Composition experiment, training a single model takes about 1 hour.
For the Effect of Memorization experiment, training a single model also takes around 1 hour.

\subsection{Memorization Score via Influence Function}\label{appendix:mem score}
Although throughout this paper we directly use the training loss to measure memorization, an alternative approach is to estimate memorization using influence functions, as proposed in \cite{feldman2020neural}. In this section, we show that long-tail data indeed exhibit high memorization scores when measured by influence functions.

We implemented the influence estimation method \citep{feldman2020neural} to quantify how training examples affect model performance. Since our task is a regression problem that predicts digit sums (rather than a classification task), we adapted their influence score by replacing the accuracy difference with a loss difference. Specifically, the memorization score is computed as:
    $$\text{mem}(i) = \mathbb{E}[\text{loss}(x_i) \mid  \text{data $i \notin$ training set}] - \mathbb{E}[\text{loss}(x_i) \mid \text{data $i \in$ training set}].$$
    A positive memorization score indicates that including the sample in training reduces its loss, suggesting a stronger memorization of the sample.

    \textbf{Experimental Setup:} We trained 100 models on randomly sampled subsets (80\% of the full training set, i.e., 32,000 out of 40,000 samples) and evaluated the memorization scores for all instances of the long-tail digit 9. The setup is the same as in Section~\ref{sec: resnet mnist}.
    
\begin{table}[th]
    \centering
    \caption{Top-10 most memorized digit-9 examples. All highly memorized examples appear exactly once in the full training set.}
    \label{tab:top10}
    \begin{tabular}{@{}ccc@{}}
    \toprule
    \textbf{Original Index Containing `9'} & \textbf{Memorization Score} \\
    \midrule
    4362 & 80.38 \\
    21445 & 80.28  \\
    9390 & 61.33  \\
    21302 & 55.62  \\
    19494 & 55.50  \\
    29286 & 49.01  \\
    31456 & 38.85  \\
    7620 & 35.90  \\
    19412 & 33.68  \\
    8847 & 32.67  \\
    \bottomrule
    \end{tabular}
    \end{table}

The result is shown in Table~\ref{tab:top10}. Our analysis shows substantial memorization of long-tail examples, with an overall mean memorization score of \textbf{2.41} (std 7.54) and a maximum of \textbf{80.38}. These values are much larger than our typical training loss, indicating that whether a specific long-tail example appears in the training set has a strong effect on OOD generalization.

\end{document}